\documentclass[twoside,11pt]{article}

\usepackage[T1]{fontenc}
\usepackage[utf8]{inputenc}
\usepackage[preprint,nohyperref]{jmlr2e}
\usepackage{amsmath}
\usepackage{booktabs}
\usepackage{algorithm}
\usepackage{algpseudocode}
\usepackage{xcolor}
\usepackage{colortbl}
\usepackage{tikz}
\usetikzlibrary{positioning,arrows.meta,fit}
\usepackage{pgfplots}
\pgfplotsset{compat=1.16}
\usepackage{lastpage}
\usepackage[colorlinks=false,allbordercolors={1 1 1}]{hyperref}
\hypersetup{
  hidelinks,
  pdftitle={MaSRead: Content-Addressed Reading of Replicated Latent Stores},
  pdfauthor={Carlos Baquero, Luis Brito, Joao Resende},
  pdfkeywords={latent communication, multi-agent systems, key-value caches,
    content addressing, conflict-free replicated data types}
}

\newcommand{\MaSRead}{MaSRead}

\title{\MaSRead{}: Content-Addressed Reading of Replicated Latent Stores}

\author{\name Carlos Baquero \email cbm@fe.up.pt \\
  \addr FEUP \& INESC TEC, Universidade do Porto, Porto, Portugal
  \AND
  \name Lu\'{\i}s Brito \email britoluis@estg.ipvc.pt \\
  \addr ESTG, Instituto Polit\'ecnico de Viana do Castelo,
  Viana do Castelo, Portugal
  \AND
  \name Jo\~ao Resende \email jresende@fc.up.pt \\
  \addr DCC, FCUP, Universidade do Porto, Porto, Portugal}

\jmlrheading{0}{2026}{1--\pageref{LastPage}}{}{}{}
  {Carlos Baquero, Lu\'{\i}s Brito, and Jo\~ao Resende}
\ShortHeadings{\MaSRead{}: Reading Replicated Latent Stores}
  {Baquero, Brito, and Resende}
\firstpageno{1}

\begin{document}
\maketitle

\begin{abstract}
Independent agents that reason in latent space can share computed state as
key--value cache fragments rather than text. Merged by a conflict-free replicated
data type, these fragments form a store that converges under any delivery order or
duplication. Yet a later query, unknown at encode time, cannot reliably read the
merged cache: colocated fragments interfere, so colocation is not addressability.
\MaSRead{} addresses the read to content. It routes through opaque keyed tag sets
derived from fragment words and decodes each selected fragment under a hard
attention mask that hides the rest. Under lexical connectivity, a graph walk reaches
the fragments required by a multi-hop query. Across chain, pipeline, symmetric,
hub, and natural-language stores, \MaSRead{} recovers visited fragments in
isolation, remains effective as unrelated fragments accumulate, and transfers to
another model family. After routing, materialized decoding depends on fragment
length rather than total store size; end-to-end work still includes store-dependent
routing and one read per visited fragment. The limits are explicit: lexical routing
can miss disconnected evidence, and answer composition remains bounded by the
frozen reader. Thus a replicated latent store becomes selectively readable for later
queries when the needed fragments connect to the query through content.
\end{abstract}

\begin{keywords}
latent communication, multi-agent systems, key-value caches, content addressing,
conflict-free replicated data types
\end{keywords}

\section{Introduction}
\label{sec:intro}

A growing line of work has language-model agents reason and communicate in
latent space rather than in text: instead of exchanging natural-language
messages, each agent encodes its input into the transformer's key--value (KV)
cache and passes that latent state on
\citep{yu2026latentspacefoundationevolution, zou2025latentmas,
jin2026agentprimitives}.
When many agents contribute, their latent states accumulate into a shared
store, a replicated collection of KV fragments, each the distilled product of
one agent's reasoning over its own input (Figure~\ref{fig:system}). An agent
monitoring a long stream of routine events need not forward the stream; it could
instead reason over it and contribute only the latent of the one anomaly it found,
so that the store holds distilled reasoning rather than raw text. That compact,
anomaly-only \emph{direct latent} contribution remains the paradigm's motivation,
not a property of our main experiments: their fragments are query-blind encodings
of text. We do, however, exercise the nearest implementation supported by the
present mechanism. A query-independent writer applies a fixed sensor rule to its
local log and emits a compact finding; that finding, rather than the raw log, is
encoded as a KV fragment and read only after replication.

This exercise also makes the signature interface concrete. The compact record
contains controlled descriptor words such as \emph{anomaly}, \emph{outlier}, and
\emph{spike}; an authorized writer maps those words to opaque HMAC tags, while the
timestamp and numerical measurement remain only in the KV payload and never enter
the signature. A later query reaches the fragment through a descriptor tag and
recovers the number under the mask; a timestamp-only query cannot route. On 20
frozen logs, three contamination levels, and four query wordings, the writer
extracts the exact anomalous value on $18/20$ logs, the routed masked read recovers
it on $216/240$ query instances, and final composition returns it on $208/240$.
Raw-log and timestamp-only controls never route. A separate single-log probe shows
retention and salience in an undistilled query-blind cache
(Appendix~\ref{app:anomaly}).

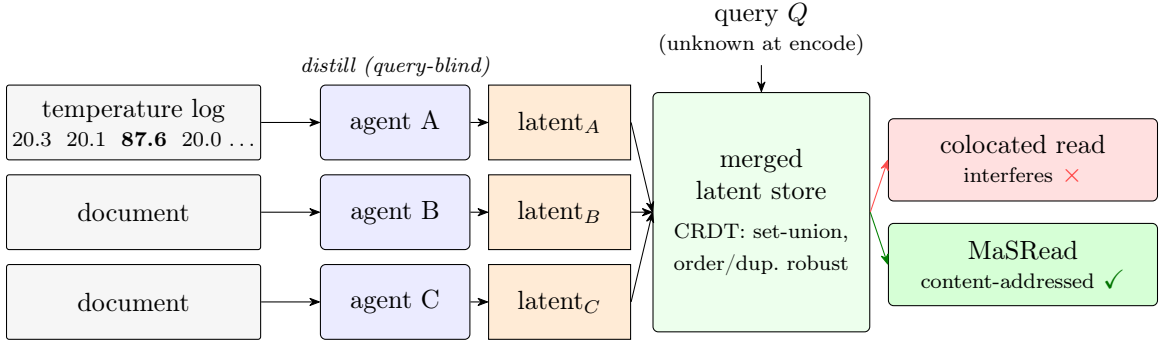
\begin{figure}[t]
\centering
\resizebox{\textwidth}{!}{%
\begin{tikzpicture}[
  font=\small, >={Stealth[round]},
  input/.style={draw, rounded corners=1pt, align=center, minimum width=34mm,
    minimum height=10mm, inner sep=2pt, fill=gray!7},
  agent/.style={draw, rounded corners=2pt, align=center, minimum width=20mm,
    minimum height=10mm, inner sep=2pt, fill=blue!8},
  lat/.style={draw, align=center, minimum width=19mm, minimum height=10mm,
    inner sep=2pt, fill=orange!16},
  store/.style={draw, rounded corners=2pt, align=center, fill=green!7,
    minimum width=29mm, minimum height=32mm},
  read/.style={draw, rounded corners=2pt, align=center, minimum width=36mm,
    minimum height=11mm, inner sep=2pt},
]
\node[input] (in1) at (0,1.2) {temperature log\\[-1pt]{\scriptsize 20.3\ \ 20.1\ \ \textbf{87.6}\ \ 20.0\ $\ldots$}};
\node[input] (in2) at (0,0) {document};
\node[input] (in3) at (0,-1.2) {document};
\node[agent] (aA) at (3.5,1.2) {agent A};
\node[agent] (aB) at (3.5,0) {agent B};
\node[agent] (aC) at (3.5,-1.2) {agent C};
\node[lat] (lA) at (5.7,1.2) {latent$_A$};
\node[lat] (lB) at (5.7,0) {latent$_B$};
\node[lat] (lC) at (5.7,-1.2) {latent$_C$};
\node[store] (store) at (8.4,0)
  {merged\\ latent store\\[3pt]{\scriptsize CRDT: set-union,}\\{\scriptsize order/dup.\ robust}};
\node[align=center] (q) at (8.4,2.45) {query $Q$\\[-1pt]{\scriptsize (unknown at encode)}};
\node[read, fill=red!12] (bad) at (11.9,0.7)
  {colocated read\\[-1pt]{\scriptsize interferes}\ {\color{red!80}$\times$}};
\node[read, fill=green!16] (good) at (11.9,-0.7)
  {\MaSRead{}\\[-1pt]{\scriptsize content-addressed}\ {\color{green!45!black}\checkmark}};
\node[font=\scriptsize\itshape] at (3.5,1.95) {distill (query-blind)};
\draw[->] (in1) -- (aA);
\foreach \i/\a in {in2/aB,in3/aC} \draw[->] (\i) -- (\a);
\foreach \a/\l in {aA/lA,aB/lB,aC/lC} \draw[->] (\a) -- (\l);
\foreach \l in {lA,lB,lC} \draw[->] (\l.east) -- (store.west);
\draw[->] (q) -- (store.north);
\draw[->, red!70]         (store.east) -- (bad.west);
\draw[->, green!50!black] (store.east) -- (good.west);
\end{tikzpicture}}
\caption{The setting. Independent agents encode their inputs into KV-cache
fragments \emph{query-blind}, distilling reasoning rather than forwarding text
(in the motivating picture, agent~A contributes only the anomaly it found in a long
stream). The
fragments merge into one replicated store under a conflict-free, set-union
merge. A query posed later, unknown when the fragments were made, is answered
by reading the store: a naive colocated read interferes, while the
content-addressed masked signature read (\MaSRead{}) routes to lexically connected
fragments, reads them in isolation, and composes an answer.}
\label{fig:system}
\end{figure}

Such a store is only useful if it can be read, and the reading problem is
awkward in a way ordinary retrieval is not. The fragments are encoded
\emph{query-blind}, before any question is posed, so a query that arrives later
must be answered from latent state that no contributing agent shaped for it.
The obvious approach is to \emph{colocate} the fragments, laying their caches
side by side and reading the concatenation~\citep{jin2026agentprimitives}. We
show this fails. Colocated
fragments interfere: a read aimed at one fragment is corrupted by the others,
and the corruption worsens as the store grows (Section~\ref{sec:failure}).
Colocation is not addressability.

We repair the read with content addressing. Each fragment carries a
content-derived \emph{signature}, an enumerable set of opaque keyed tags derived
from its words, and an authorized query client produces tags in the same domain.
The store routes by tag overlap, then decodes each selected fragment under a hard
attention \emph{mask} that admits only that fragment's block, so the others cannot
interfere. The routing is lexical, and it is the mask, not the routing, that our
control isolates: a read under the wrong mask returns the wrong fragment's value
rather than a cleaned-up answer, which is what makes the repair addressing rather
than denoising. A query is answered by a short program over the store: seed on the
query, walk a signature graph to the fragments it needs, read each under its mask,
and compose the results (Section~\ref{sec:masread}). We call this the \emph{masked signature read}
(\MaSRead{}). The store beneath it is a conflict-free replicated data type:
fragments are identified by a content hash and merged by set union, so it
converges regardless of the order or the multiplicity in which fragments arrive,
tolerating duplicate transmissions (Section~\ref{sec:store}).

\paragraph{What we learn.}
The central result is not merely that one read implementation works. Before a
fragment enters the store, \emph{payload formation} must decide what local finding
to preserve; the routed anomaly exercise shows that an incorrect writer record is
faithfully routed and cannot be repaired downstream. The store/read path then has
six separable obligations, and satisfying an earlier one does not satisfy the next.
\emph{Storage and convergence} preserve the same immutable
fragments at every replica, but do not make any one fragment selectively readable.
\emph{Routing} must find the fragments required by a later query; our opaque
signature walk does so only when a lexical path connects the query to them.
\emph{Addressing} must then expose the selected fragment without exposing its
neighbors: the wrong-mask control shows that the hard mask selects which fragment
is read, rather than merely denoising the colocated store. \emph{Recovery} asks
whether the selected cache can be decoded into the content it holds, a
model-dependent step measured by fragment restatement. \emph{Isolation} asks
whether unrelated stored fragments can alter that recovery; masking supplies
selective access empirically, and materializing a located block makes its
fragment-level decoder input structurally independent of the rest of the store.
Finally, \emph{composition} asks whether the reader can turn the recovered
fragments into the final answer. It remains a separate, reader-bounded operation:
successful routing and recovery do not imply successful reasoning over what was
recovered.

This decomposition makes an end-to-end score interpretable and is the paper's main
general lesson. Failure can mean that evidence was absent, missed by routing,
misaddressed, decoded incorrectly, contaminated by other fragments, or recovered
but not composed; the remedies differ. Our controls localize these cases:
colocation fails at addressing although the information remains recoverable; the
lexical walk fails on disconnected evidence; and some answer errors remain after
full coverage and near-exact restatement, locating the residual in composition
(Sections~\ref{sec:masread} and~\ref{sec:eval-boundary};
Appendices~\ref{app:anomaly} and~\ref{app:precision}). In the anomaly exercise,
the same accounting attributes all masked-value misses to two incorrect target
writer records, while $8/216$ correct masked recoveries are subsequently lost in
composition. Where routing holds, isolated recovery remains
effective as unrelated fragments accumulate, including on natural-language stores
(Sections~\ref{sec:eval-scale} and~\ref{sec:eval-nl}). Once a fragment is located
and materialized, its decode cost depends on that fragment rather than total store
size, although routing and the number of reads remain store- and query-dependent.

We claim addressability, robustness to contamination, and read cost, not
superiority over text retrieval. Pipelines that store and re-read text are a
different regime: they presuppose that the source text is retained, which the
setting we study does not provide. We therefore treat text retrieval as
orthogonal rather than as a baseline, and account for the trade-off, read cost
included, honestly in Section~\ref{sec:disc-textret}.

Content addressing is one way to read a replicated latent store, by what a
fragment contains; addressing by where a fragment sits in a canonical order is a
complementary direction we leave to future work. Our claim
is narrow and, we believe, consequential: once agents reason in latent space,
the store they produce can be read for queries they never saw, when those queries
are connected by content to the fragments they need, and, once a fragment is
located, at a read cost that does not grow with the store, provided it is
addressed rather than merely colocated.

\subsection{Contributions}
\label{sec:contributions}
\begin{itemize}
\item \textbf{A diagnosis.} We show that a query-blind, CRDT-merged store of
key-value fragments cannot be read by decoding over the colocated whole: the
fragments interfere and the failure worsens as the store grows, and a wrong-mask
control identifies the missing operation as addressing rather than denoising
(Sections~\ref{sec:failure} and~\ref{sec:masread}).
\item \textbf{A read mechanism.} We repair the read with the \emph{masked signature
read} (\MaSRead{}): route to a fragment by a lexical signature of its content,
decode it under a hard attention mask that isolates its block, and walk a signature
graph to reach the several fragments a multi-hop query needs
(Section~\ref{sec:masread}).
\item \textbf{An evaluation.} We measure the read query-blind across four store
structures, under contamination to many times the fragments a query needs, on
natural-language multi-hop questions, and across two model families, and we separate
the read's accuracy from the agents' latent computation. A fixed-rule sensor
exercise additionally follows a compact writer finding through opaque routing,
masked recovery, and final composition. Where the lexical walk has
coverage, the read recovers required fragments, isolates their recovery as unrelated
content accumulates, and reads a located materialized fragment at a cost independent
of store size (Section~\ref{sec:eval};
Appendices~\ref{app:anomaly}, \ref{app:scale}, \ref{app:llama},
and~\ref{app:determinism}).
\item \textbf{Two honest boundaries.} We locate rather than hide where the method
stops: the store delivers the facts reliably, but composing them into an answer is
bounded by the reader (Section~\ref{sec:eval-boundary}), and the lexical routing
fails deterministically when a required fragment shares no word with the query
(Appendix~\ref{app:precision}).
\end{itemize}

\section{The replicated latent store}
\label{sec:store}

\subsection{The replicated latent-store read problem}
\label{sec:task}

Writer $i$ receives a local input $x_i$ and, before the eventual query is known,
produces an immutable element
$e_i=(\mathrm{id}_i,\mathrm{KV}_i,T_i,m_i)$. Here $\mathrm{KV}_i$ is its cache
block, $T_i$ is the raw opaque-tag sidecar, and $m_i$ fixes the model, encoding,
layout, and tag-schema metadata needed to interpret the element. A replica holds
a set $S\subseteq\{e_i\}$ and merges received state by union. A later query $q$
has a task-dependent required set $R(q)\subseteq S$: this set is available only
for evaluation, not to the router. The router returns an ordered visited set
$V(q,S)$; a fragment reader maps each visited element to a recovered output
$r_f$; and a composer maps $(q,(r_f)_{f\in V})$ to the final answer.

This decomposition separates three questions that a single accuracy number can
confound. \emph{Addressing} asks whether a located fragment can be exposed without
admitting the others. \emph{Routing} asks whether $R(q)\subseteq V(q,S)$.
\emph{Answering} asks whether the composer can use the recovered outputs. MaSRead
supplies the first operation, implements the second with a lexical walk, and
delegates the third to the frozen reader. Table~\ref{tab:properties} records the
resulting contract.

\begin{table}[t]
\centering
\small
\begin{tabular}{@{}p{4.7cm}p{9.2cm}@{}}
\toprule
property & status in this paper \\
\midrule
convergence & structural for complete immutable elements under deterministic encoding and rendering \\
\arrayrulecolor{black!14}\specialrule{0.3pt}{1pt}{1pt}\arrayrulecolor{black}
query blindness & provided by constructing each element before the eventual query is known \\
\arrayrulecolor{black!14}\specialrule{0.3pt}{1pt}{1pt}\arrayrulecolor{black}
selective recovery & provided for a located fragment by masking or materialization \\
\arrayrulecolor{black!14}\specialrule{0.3pt}{1pt}{1pt}\arrayrulecolor{black}
contamination isolation & structural for one located, materialized fragment; visited distractors can still affect composition \\
\arrayrulecolor{black!14}\specialrule{0.3pt}{1pt}{1pt}\arrayrulecolor{black}
routing coverage & conditional on lexical connectivity, measured as $R(q)\subseteq V(q,S)$ \\
\arrayrulecolor{black!14}\specialrule{0.3pt}{1pt}{1pt}\arrayrulecolor{black}
semantic routing & not provided \\
\arrayrulecolor{black!14}\specialrule{0.3pt}{1pt}{1pt}\arrayrulecolor{black}
final-answer composition & delegated to and bounded by the frozen reader \\
\arrayrulecolor{black!14}\specialrule{0.3pt}{1pt}{1pt}\arrayrulecolor{black}
located-fragment read cost & depends on the fragment length, not $|S|$, after materialization \\
\arrayrulecolor{black!14}\specialrule{0.3pt}{1pt}{1pt}\arrayrulecolor{black}
end-to-end query cost & includes store-dependent routing and one read per visited fragment \\
\arrayrulecolor{black!14}\specialrule{0.3pt}{1pt}{1pt}\arrayrulecolor{black}
fully latent read & not provided: endpoints use lexical metadata and recovery uses text restatements \\
\bottomrule
\end{tabular}
\caption{The replicated latent-store read contract. Structural statements concern
state or decoder inputs, not model accuracy; empirical conditions are measured in
Section~\ref{sec:eval}.}
\label{tab:properties}
\end{table}

\subsection{Store state and convergence}

We now make the store precise, because its structure is what makes the read
problem hard. Each agent, having encoded its input query-blind under one shared,
frozen model, contributes a \emph{fragment}: the key--value cache of that encode
and an immutable lexical-addressing sidecar, named together by a content identifier.
The sidecar contains opaque pseudorandom tags rather than words; its construction is
defined in Section~\ref{sec:masread}. Because the model, normalization schema, and encoding
are fixed and deterministic, the element and its name are a function of the input alone: two
agents that encode the same input produce the same fragment under the same name,
and two that encode different inputs produce distinct fragments, so identity is
exact rather than semantic. The identifier binds the cache, sidecar, schema and
public key identifier, so two different sidecars cannot occupy the same element
identity. The state of the store is the set of fragments it has
received, and it records which fragments are present and nothing else: not when a
fragment arrived, nor in what company, nor how many times.

Replicas of the store synchronize by exchanging and merging states, and the
merge is set union. Merging a fragment already present changes nothing, and two
replicas hold the same set whatever the order in which they saw its fragments, so
the merge is commutative, associative, and idempotent. The store is thus a
grow-only set, the simplest state-based conflict-free replicated data type: under
eventual delivery, replicas that gossip and merge states converge to the union of
all fragments ever added, independent of the order, grouping, or multiplicity of
delivery \citep{shapiro2011crdt,preguica2018crdts}. Fragment identity is by
content hash, so convergence is exact up to hash collisions, which we take to be
negligible.

To be read, the set is \emph{rendered} into a single cache
by placing each fragment's block at a canonical range of positions, which fixes
its rotary (RoPE) coordinates; because that placement is a deterministic
function of the set, the render is byte-identical across delivery orders, a
construction we take from prior work on convergent cache merging and
describe in Appendix~\ref{app:determinism} \citep{baquero2026cachemerging}.

\begin{proposition}[Convergence of complete elements]
\label{prop:convergence}
Assume that writers use the same model, tokenizer, normalization and tag schema,
that encoding and rendering are deterministic, that elements are immutable and
completely named by their payload identity, and that hash collisions are
negligible. Replicas that receive the same elements hold the same state
independently of delivery order, grouping, and duplication. For a fixed tagged
query, deterministic rendering and tag traversal with the content-derived
tie-break also produce the same render and visit order.
\end{proposition}
\begin{proof}
Set union is commutative, associative, and idempotent, so every delivery history
with the same delivered elements yields the same $S$. The boilerplate
intersection, final tag sets, canonical layout, overlap scores, and content-hash
tie-break are deterministic functions of $(S,q)$; applying them to equal inputs
therefore yields equal renders and walks. The proposition concerns agreement, not
whether the walk reaches $R(q)$ or whether the model decodes correctly.
\end{proof}

One property of this state shapes everything that follows: it is a \emph{set},
identified by content. Rendering it does place each fragment at a position in the
layout, but that position is fixed by the layout, not by the query, so a query
has no position to ask for; its only handle is content. A query, which the
fragments were not encoded to answer, must therefore be answered by what the
fragments contain. The rendered cache is exactly the colocated store the next
section reads: reading it whole, we will see, affords no such handle, and
Section~\ref{sec:masread} recovers one by reading the same render under a mask.

\section{Colocation is not addressability}
\label{sec:failure}

The store of Section~\ref{sec:store} is read by rendering its fragments into one
colocated cache and decoding over the whole. We now show that this read fails,
and fails worse as the store grows. Figure~\ref{fig:kramp} takes a store of $k$
same-family fragments, renders them together, and asks, query-blind, for the
value held by one designated fragment. At $k=2$ the colocated read is already
unreliable, at $0.43$; by $k=8$ it has collapsed to $0.00$. The fragments are
not lost: the same store, read one fragment at a time under a mask to an
oracle-designated target block
(Section~\ref{sec:masread}), answers every query, at $0.99$ or above for every $k$. The
information is present in the store, but reading it whole does not give reliable
selective access to it, and the loss grows with the store. Here the target
fragment is oracle-designated solely to isolate recovery; the figure is not
evidence that the lexical walk finds it. Query-driven routing is evaluated later
(Section~\ref{sec:eval}).

\begin{figure}[t]
\centering
\begin{tikzpicture}
\begin{axis}[
  width=0.80\textwidth, height=5cm,
  xlabel={store size $k$ (colocated fragments)}, ylabel={read accuracy},
  xmode=log, log basis x={2}, xtick={2,4,8,16}, xticklabels={2,4,8,16},
  ymin=-0.05, ymax=1.07, ytick={0,0.25,0.5,0.75,1},
  legend pos=south east, legend cell align={left},
  grid=major, grid style={gray!16},
  tick label style={font=\small}, label style={font=\small},
  legend style={font=\small},
]
\addplot[red!75, thick, mark=*, error bars/.cd, y dir=both, y explicit]
  table[x=k, y=acc, y error plus=ep, y error minus=em] {
    k acc ep em
    2 0.43 0.098 0.093
    4 0.08 0.070 0.039
    8 0.00 0.037 0.000
    16 0.00 0.037 0.000
  };
\addlegendentry{colocated read}
\addplot[green!45!black, thick, dashed, mark=square*, error bars/.cd, y dir=both, y explicit]
  table[x=k, y=acc, y error plus=ep, y error minus=em] {
    k acc ep em
    2 1.00 0.000 0.037
    4 1.00 0.000 0.037
    8 1.00 0.000 0.037
    16 0.99 0.008 0.044
  };
\addlegendentry{designated masked read}
\end{axis}
\end{tikzpicture}
\caption{Colocation is not addressability. A store of $k$ same-family fragments
is rendered together and read, query-blind, for the value in one designated
fragment ($n=100$ per point; bars are $95\%$ Wilson intervals). The target block is
supplied by an oracle in both conditions, so this diagnostic tests isolation, not
routing. The colocated read collapses as the store grows, while the same fragments
read one at a time under a mask to the target block stay at or near $1.00$. The gap
is the cost of not addressing.}
\label{fig:kramp}
\end{figure}
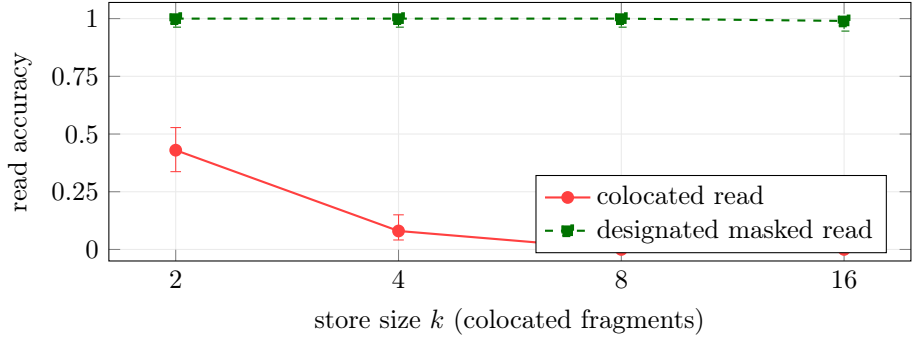

The failure is interference, not loss. Reading over the colocated cache, a read
aimed at one fragment is not confined to it: the answer it commits to is drawn
from the wrong fragment or fused across several. Asked for the value in fragment
$a$, the colocated read returns the value that sits in fragment $b$; asked to
apply one fragment's rule, it applies another's. We name this by the substituted
answers we observe, not by a measured attention pattern, so we call it
interference without asserting a specific route through attention. The more
fragments share the store, the more neighbours there are to confuse it, and the
read collapses within a few fragments rather than degrading gently.

This is the concrete form of the mismatch named in Section~\ref{sec:store}. The
store is a set identified by content, with no position a query can name; the
render supplies positions, but reading across them is what corrupts the answer.
To read the store is therefore to \emph{isolate} the fragment a query wants from
the rest, which colocation does the opposite of. The next section builds that
isolation, and shows with a control that what it restores is addressing rather
than a merely cleaner read.

Two qualifications frame everything that follows, because the failure above is a
worst case. It is sharpest when the fragments are lexically similar, as here, and
when the store is large; on a small store of dissimilar fragments a colocated read
can be adequate, and on clean natural text it can match an addressed read, as the
evaluation will show (Section~\ref{sec:eval}). The value of addressing is not that
colocation always fails but that, once a target is located, its isolated recovery
does not admit neighbouring fragments as similarity, store size, or contamination
changes, where a colocated read does.
That invariance is bought at a price, the read's dependence on routing recall: a
fragment a query cannot reach by content is a fragment the addressed read cannot
use, and there a colocated read that keeps everything in view can do better
(Appendix~\ref{app:precision}). The evaluation reports both the invariance and its
price.

\section{Content-addressed reads: the masked signature read}
\label{sec:masread}

Section~\ref{sec:failure} left the read needing to isolate the fragment a query
wants from the rest of the colocated store. We supply that isolation by addressing
the read to content, in three operations of which only the last is new here.
Fragments are identified by content hash (Section~\ref{sec:store}), which merges
and deduplicates the store but is not a handle a query holds. A query instead
\emph{routes} to fragments by a lexical \emph{signature}, and each located fragment
is then \emph{read} under a hard attention mask.

A writer normalizes its fragment text locally: lowercase alphabetic content words,
simple stemming, function-word removal, and no digits. Under a store-scoped secret
key $K$, each normalized word $w$ becomes the first $128$ bits of a domain-separated
$\operatorname{HMAC\mbox{-}SHA256}(K,d\mathbin\Vert w)$, where $d$ is the ASCII
prefix \texttt{masread/signature/v1} followed by a zero byte. The writer sends the resulting enumerable
tag set with the cache and may then discard the source text. Authorized query
clients hold the same key and tag their query locally; storage replicas need
neither the key nor any plain signature word. Tags are pseudorandom identifiers,
but equality is preserved within one key epoch.

The transmitted set is the fragment's raw tag set. From the convergent element set
$S$, every replica removes the set-wide boilerplate $B_S$, the intersection of all
raw tag sets, and uses $\mathrm{sig}_S(f)=T_f\setminus B_S$ as the final signature.
The set must remain enumerable because the walk of Algorithm~\ref{alg:masread}
grows its frontier by unioning a visited fragment's tags. A $128$-bit Bloom filter
over the final set may accelerate membership tests but cannot replace it. On the
synthetic families digits never enter the tagger, so the numerical answers are not
in the addressing metadata. On natural-language paragraphs an answer word may be
tagged, but a keyless replica sees only its stable opaque identifier.

A query is tagged into the same domain, and the read routes to the fragments whose
signatures its tags select, preferring the largest overlap. This
is lexical routing, a relevance-ranked walk rather than an exact address lookup,
and we claim no more for it. Having located a fragment, the read decodes over the
render under a hard attention \emph{mask} that admits only that fragment's block
and hides every other, so the interference of Section~\ref{sec:failure} cannot
arise: the model attends to one fragment and reads it as if it stood alone. The
control in Table~\ref{tab:mech} isolates this masking step; routing and
composition are measured separately in Section~\ref{sec:eval}.

A query is answered by a short program over the store (Algorithm~\ref{alg:masread}).
It \emph{seeds} a frontier with the query's tags and matches them, by set
membership, against the fragments' signatures; each matched fragment is visited,
and its own signature tags are unioned into the frontier, so the search expands
without ever enumerating a Bloom filter. A fragment reached
only through a tag that another fragment supplied is exactly a second hop. The
relevance-prioritized graph walk continues over edges that join fragments sharing a
lexical tag until no unvisited signature overlaps the frontier. Each visited fragment is then \emph{read} under its mask,
and the recovered facts are \emph{composed} into an answer. We call the whole a
\emph{masked signature read} (\MaSRead{}).

\begin{algorithm}[t]
\caption{The masked signature read (\MaSRead{}).}
\label{alg:masread}
\begin{algorithmic}[1]
\Require tagged query $T_q$; store where each fragment $f$ carries raw tag set $T_f$
\State $B \gets \bigcap_f T_f$; $\mathrm{sig}(f)\gets T_f\setminus B$
\State $F \gets T_q$ \Comment{frontier of opaque lexical tags}
\State $V \gets [\,]$ \Comment{visited fragments, in order}
\While{some unvisited $f$ has $\mathrm{sig}(f) \cap F \neq \emptyset$} \Comment{a set-membership test}
  \State $f \gets$ the unvisited fragment of largest overlap with $F$
  \State append $f$ to $V$; \quad $F \gets F \cup \mathrm{sig}(f)$ \Comment{grow the tag frontier}
\EndWhile
\State \Return $\textsc{Compose}\big(\textsc{MaskedRead}(f) \text{ for each } f \in V\big)$
\Statex \hspace{1.2em}\emph{\footnotesize$\textsc{MaskedRead}(f)$: decode $f$'s block in the render under a hard mask that hides every other block.}
\end{algorithmic}
\end{algorithm}

\begin{proposition}[Isolation of a materialized fragment read]
\label{prop:isolation}
Fix a located element $f$. Let materialization extract $f$'s KV block from the
canonical render and reverse its layout offset, restoring the native coordinates
at which it was encoded. A decoder supplied only that materialized block has a
fragment-level input independent of every element in $S\setminus\{f\}$.
Consequently, adding, removing, or changing other elements cannot change this
fragment-level computation.
\end{proposition}
\begin{proof}
After extraction and coordinate restoration, the decoder inputs are a function
only of $f$, the fixed model, and the fixed restatement prompt. No key or value
tensor from another element is present. The statement makes no accuracy claim and
does not cover routing or composition: changing $S\setminus\{f\}$ may change
whether $f$ is visited, may add other visited restatements, and may thereby change
the final answer.
\end{proof}

The tags hide plain signature words from a replica without $K$, not the structure
of the index. A replica observes repeated tags, signature sizes, co-occurrence,
the induced graph, repeated query tags, and which fragments a query visits; a key
holder can also test dictionary words. This is the familiar leakage trade-off of
searchable symmetric indexes \citep{curtmola2011sse,cash2015leakage}, not
zero-knowledge search. We use one key shared by authorized writers and query
clients. An oblivious pseudorandom function could keep that key from mutually
distrustful writers, at the cost of an online cryptographic service, which we
leave to future work \citep{miao2020psi}.

\begin{table}[t]
\centering
\small
\begin{tabular}{lc}
\toprule
read & accuracy \\
\midrule
open (colocated) & 0.625 \\
masked (target block) & 1.000 \\
wrong-mask control (partner's block) & 0.000 \\
\bottomrule
\end{tabular}
\caption{Mechanism panel: Qwen3-1.7B, a two-fragment store read query-blind for
the value in one designated fragment ($n=40$, at the shared-unit similarity rung; a
separate draw from Figure~\ref{fig:kramp}, whose open-read level varies with draw
and rung, Appendix~\ref{app:precision}). The colocated read interferes; masking
to the target block restores retrieval; a read under the \emph{wrong} mask
returns the partner fragment's value, at $0.000$ correct. The mask selects which
fragment is read. The same ordering holds at 4B and 8B
(Appendix~\ref{app:scale}).}
\label{tab:mech}
\end{table}

Table~\ref{tab:mech} isolates the mechanism on the smallest store that
interferes: two same-family fragments, read query-blind for the value in one of
them. The open read over the pair scores $0.625$; masking to the target block
restores it to $1.000$. The decisive row is the control: a read under the
\emph{wrong} mask, admitting the partner's block instead of the target's, scores
$0.000$, and its answers are the partner's value. This is what makes the repair
\emph{addressing} rather than denoising. The mask does not clean up a noisy read;
it selects which fragment is read, and under the wrong selection the model
faithfully returns the wrong fragment. Read accuracy is a function of which block
the mask admits, which is what it means for the mask to be an address, and it
recovers the isolation Section~\ref{sec:failure} showed colocation destroys.
Both this panel and Figure~\ref{fig:kramp} designate the target block, so they
measure the mask on its own; whether the walk of Algorithm~\ref{alg:masread}
routes to that fragment from a query it never saw is a separate question, and the
coverage it reaches is measured throughout Section~\ref{sec:eval}.

A word on what these probes contain: the fragments hold explicit facts, so the
read can be graded against a gold value; the distilled-reasoning premise of
Section~\ref{sec:intro}, where a fragment holds the product of an agent's
reasoning rather than a stated fact, is exercised separately in
Appendix~\ref{app:anomaly}. Two properties of the read we also defer. The
signature is lexical, a summary of content words, which fixes a precision
boundary we return to in Section~\ref{sec:disc-textret}. And the masked read
above reads each block at its position in the render; a materialized variant
reads it at a cost independent of store size, which we use in
Section~\ref{sec:eval-scale}.

\section{Evaluation}
\label{sec:eval}

Sections~\ref{sec:failure} and~\ref{sec:masread} measured the mask with a
designated target. We now measure the full read, routing included, end to end:
each store is read query-blind by Algorithm~\ref{alg:masread} with no oracle
anywhere, so the walk must find the fragments from the query's lexical tags, read them
under masks, and compose. Alongside accuracy we report two diagnostics that
localize any failure: \emph{coverage}, the fraction of a query's required
fragments the walk actually reaches (routing), and whether each reached fragment
is restated correctly (recovery). A gap that survives both is composition, in the
reader rather than the store.

\begin{table}[t]
\centering
\footnotesize
\begin{tabular}{@{}p{4.6cm}p{4.0cm}p{5.3cm}@{}}
\toprule
claim & evidence & boundary or control \\
\midrule
colocated fragments interfere & Figure~\ref{fig:kramp} and Table~\ref{tab:ladder} & similarity ladder shows the effect is strongest for lexically similar fragments \\
\arrayrulecolor{black!14}\specialrule{0.3pt}{1pt}{1pt}\arrayrulecolor{black}
the mask acts as an address & Table~\ref{tab:mech} & wrong-mask control returns the partner rather than a denoised target \\
\arrayrulecolor{black!14}\specialrule{0.3pt}{1pt}{1pt}\arrayrulecolor{black}
query-driven routing can reach multi-hop evidence & coverage in Tables~\ref{tab:headline} and~\ref{tab:nl} & natural-text misses and the \textsc{discount} family expose the lexical-recall boundary \\
\arrayrulecolor{black!14}\specialrule{0.3pt}{1pt}{1pt}\arrayrulecolor{black}
isolated recovery survives irrelevant accumulation & Figure~\ref{fig:sweep} and Table~\ref{tab:nl} & topical distractors are visited and their separate restatements can still affect composition \\
\arrayrulecolor{black!14}\specialrule{0.3pt}{1pt}{1pt}\arrayrulecolor{black}
a located materialized read is store-size independent & Proposition~\ref{prop:isolation} and Table~\ref{tab:rebased} & fixed decoded-block length, not an end-to-end latency result \\
\arrayrulecolor{black!14}\specialrule{0.3pt}{1pt}{1pt}\arrayrulecolor{black}
recovery and composition are distinct & Table~\ref{tab:boundary} & symbolic composer controls for whether the read output determines the answer \\
\arrayrulecolor{black!14}\specialrule{0.3pt}{1pt}{1pt}\arrayrulecolor{black}
writer finding, routing, recovery, and composition are distinct & Table~\ref{tab:routed-anomaly} & frozen fixed-rule sensor evaluation retains writer false positives, misses, and truncated compositions \\
\bottomrule
\end{tabular}
\caption{Claim--evidence map. Each claim is paired with the control or failure
boundary that limits its scope. The routed anomaly exercise is a fixed-rule
computed-payload example, not unsupervised anomaly discovery.}
\label{tab:evidence-map}
\end{table}

\subsection{Breadth across store structures}
\label{sec:eval-breadth}

We evaluate on synthetic stores of four families, chosen to span distinct
structures; Table~\ref{tab:families} shows one instance of each. Two are
sequential: a \emph{chain} of fictional-unit conversions and an affine
\emph{pipeline} of machines, each fragment's value feeding the next. One is
\emph{symmetric}, a constraint system whose three fragments are solved jointly
with none ordering the others. The last is a \emph{hub}, a single dense fragment
indexed against two small tables. Every store has three fragments, encoded
query-blind (construction in Appendix~\ref{app:repro}).

\begin{table}[t]
\centering
\small
\renewcommand{\arraystretch}{1.25}
\begin{tabular}{@{}l p{6.1cm} p{2.4cm} c@{}}
\arrayrulecolor{black}
\toprule
family & the three fragments (encoded query-blind) & query & answer \\
\midrule
chain & ``1 wog $=$ 3 wibbles''; ``1 wibble $=$ 2 blargs''; ``the race is 24 blargs long'' & how many wogs is the race? & $4$ \\
\arrayrulecolor{black!14}\specialrule{0.3pt}{1pt}{1pt}\arrayrulecolor{black}
pipeline & ``machine one: $\times 3$, then $+2$''; ``machine two: $\times 2$, then $+1$''; ``start with 5 widgets'' & output of machine two? & $35$ \\
\arrayrulecolor{black!14}\specialrule{0.3pt}{1pt}{1pt}\arrayrulecolor{black}
symmetric & ``Alice is $3\times$ Bob's age''; ``Bob is 2 years older than Carol''; ``together they are 23'' & how old is Alice? & $15$ \\
\arrayrulecolor{black!14}\specialrule{0.3pt}{1pt}{1pt}\arrayrulecolor{black}
hub & ``base: red 2, blue 3, green 4''; ``size: small $\times1$, large $\times2$''; ``collected 2 large red, 3 small green, 1 large blue'' & total points? & $26$ \\
\bottomrule
\end{tabular}
\caption{One instance of each store family ($k=3$ fragments, abbreviated from the
generator templates; the generators randomize the words and numbers). The chain
and pipeline pass one value along a sequence; the symmetric system couples all
three fragments; the hub packs several items into its third fragment, each to be
cross-referenced against the other two. Only the hub reads below ceiling
(Table~\ref{tab:headline}, Section~\ref{sec:eval-boundary}).}
\label{tab:families}
\end{table}

Table~\ref{tab:headline} reads each store query-blind, pooled over three data
seeds ($n=300$ per cell).

\begin{table}[t]
\centering
\small
\begin{tabular}{lcccccc}
\toprule
store & masked read & ceiling & colocated & no-protocol & coverage & hub \\
\midrule
chain              & $0.97\,[.94,.98]$ & $1.00$ & $0.93$ & $0.08$ & $1.00$ & -- \\
pipeline           & $1.00\,[.99,1.0]$ & $1.00$ & $0.90$ & $0.04$ & $1.00$ & -- \\
symmetric          & $0.99\,[.97,1.0]$ & $0.99$ & $0.98$ & $0.02$ & $1.00$ & -- \\
hub                & $0.44\,[.39,.50]$ & $0.92$ & $0.19$ & $0.04$ & $1.00$ & $1.00$ \\
\bottomrule
\end{tabular}
\caption{The masked signature read across four store structures ($k{=}3$;
Qwen3-1.7B; query-blind; pooled over three data seeds, $n=300$ per cell).
\emph{masked read} is \MaSRead{} end to end; \emph{ceiling} is the full-text
single-context answer, a reference a masked read can match or exceed rather than an
upper bound; \emph{colocated} reads the same extracted facts alongside the whole merged render
rather than in isolation, the naive colocation of Section~\ref{sec:failure};
\emph{no-protocol} poses the query over the merged render alone, with neither the
addressing nor the extracted facts;
\emph{coverage} is the fraction of required fragments the walk reaches; \emph{hub} is
the multi-value readout (hub store only). The bracketed interval on the masked read
is the $95\%$ Wilson binomial (Appendix~\ref{app:stats}).}
\label{tab:headline}
\end{table}

The read reaches the full-text ceiling on three of the four families: the chain,
the pipeline, and the symmetric store all land at $0.97$ or above, against
ceilings at or near $1.00$, so this is not one lucky structure.
Two diagnostics say why. Coverage is $1.00$ everywhere: the walk reaches every
required fragment from the query alone, which is where the routing deferred in
Sections~\ref{sec:failure} and~\ref{sec:masread} is validated, and the
near-ceiling reads are content-addressed end to end, not oracle-masked. And the
no-protocol control sits at floor, never above $0.08$: the colocated store holds
the same content, yet read without addressing it yields almost nothing, so the
read's accuracy is addressing, not extra information. On every structure the
colocated read is below the masked read, replaying the interference of
Section~\ref{sec:failure}.

The hub is the exception, at $0.44$. It is not a routing or a reading failure:
coverage and the hub's multi-value readout are both $1.00$, so the store returns
every fact the query needs. What the frozen reader then fails is composing those
facts into the answer, a read-versus-answer boundary we take up in
Section~\ref{sec:eval-boundary}.

\subsection{Scaling with contributing agents}
\label{sec:eval-scale}

A shared store grows as agents contribute, and most of what accumulates is
irrelevant to any one query. We test whether the read survives this: we take a
store that reads well, merge in unrelated fragments drawn from a disjoint family
that shares no words with the query, and read for the original target as the
store fills. Figure~\ref{fig:sweep} sweeps the number of unrelated fragments $D$
from $0$ to $64$, more than twenty times the fragments a query needs.

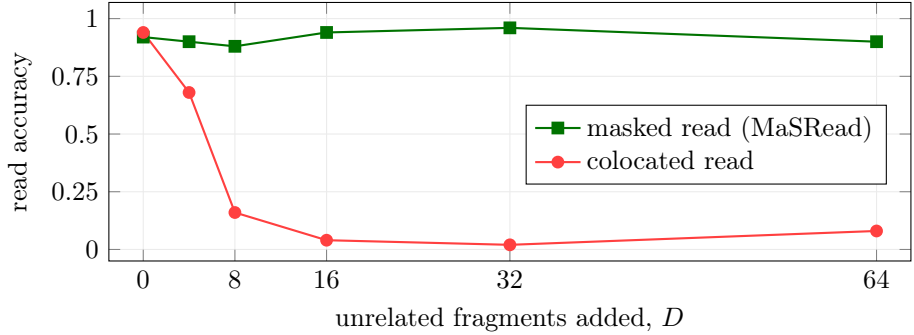
\begin{figure}[t]
\centering
\begin{tikzpicture}
\begin{axis}[
  width=0.80\textwidth, height=5cm,
  xlabel={unrelated fragments added, $D$}, ylabel={read accuracy},
  xmin=-3, xmax=67, xtick={0,8,16,32,64},
  ymin=-0.05, ymax=1.07, ytick={0,0.25,0.5,0.75,1},
  legend style={at={(0.97,0.45)}, anchor=east, font=\small}, legend cell align={left},
  grid=major, grid style={gray!16},
  tick label style={font=\small}, label style={font=\small},
]
\addplot[green!45!black, thick, mark=square*]
  coordinates {(0,0.92)(4,0.90)(8,0.88)(16,0.94)(32,0.96)(64,0.90)};
\addlegendentry{masked read (\MaSRead{})}
\addplot[red!75, thick, mark=*]
  coordinates {(0,0.94)(4,0.68)(8,0.16)(16,0.04)(32,0.02)(64,0.08)};
\addlegendentry{colocated read}
\end{axis}
\end{tikzpicture}
\caption{Scaling under contamination. A chain store is read for its target as $D$
unrelated fragments, sharing no query word, are merged in ($n=50$ per point). The
masked read stays near $0.90$ out to $D=64$, over twenty times the fragments a
query needs; the colocated read over the same store collapses. Coverage is
$1.00$ and the walk visits none of the $D$ unrelated fragments at every point.}
\label{fig:sweep}
\end{figure}

The masked read holds near $0.90$ across the whole sweep, undiminished as the store
fills with unrelated fragments to twenty times what a query needs. The colocated read over the same store
collapses, from $0.94$ to near zero by $D=16$, a collapse consistent with the
long-context degradation a colocated read is subject to \citep{liu2024lostinmiddle}.
Why the addressed read does not move is visible in the walk: coverage stays at $1.00$ and it visits \emph{zero}
of the unrelated fragments at every $D$, because a fragment that shares no word
with the query has an empty signature intersection and never becomes a candidate.
Contamination cannot reach the read: the walk refuses the junk, and each masked
read isolates its target regardless of what else the store holds.

Cost separates into routing, per-fragment recovery, and composition:
\begin{equation}
\label{eq:cost}
T_{\mathrm{total}}(q,S)
=T_{\mathrm{route}}(q,S)
+\sum_{f\in V(q,S)}T_{\mathrm{read}}(L_f)
+T_{\mathrm{compose}}(|V(q,S)|).
\end{equation}
Only the reading of a single located fragment is independent of the store. The
full read is not: the walk scans the signatures and visits every fragment that
shares a frontier tag, so $T_{\mathrm{route}}$ generally grows with $|S|$ and,
on a lexically connected store, $|V|$ can approach $|S|$. Every visited fragment
is then read and included in composition. What does not grow is the cost of reading
one already located fragment. An in-place masked read does not have this
property, because it attends over the whole render, whose length grows with every
fragment merged; a \emph{materialized} read does: the target block is extracted and
decoded at the native positions it was encoded at (Section~\ref{sec:store}), so its
cost is the block's own length $O(L_f)$, independent of the store size $|S|$
(Appendix~\ref{app:determinism} measures a fixed $147$-token decode as the render
grows roughly eightfold). The accuracy results above use the in-place read; the
materialized read is the same read relocated to fixed coordinates, and it is what
the read-cost claim of Section~\ref{sec:intro} refers to, once a fragment is
located. We report decoded-block length, not wall-clock latency, so this is an
asymptotic argument about the located read, not an end-to-end timing result.

\subsection{Answering unknown queries on real text}
\label{sec:eval-nl}

The synthetic stores let us control structure; they do not tell us whether the
read survives the lexical noise of real language. We repeat the evaluation on two
standard multi-hop question-answering sets: MuSiQue 2-hop \citep{trivedi2022musique},
the answerable subset, and HotpotQA bridge \citep{yang2018hotpotqa}, $k=2$. Each
item's supporting paragraphs become the fragments,
encoded query-blind, and we read for the answer the encoding never saw. To probe
contamination we add $d$ of the dataset's own distractor paragraphs, which are
\emph{topical}: they are about the same entities and share vocabulary with the
query, so unlike the disjoint junk of Section~\ref{sec:eval-scale} the walk has no
lexical reason to skip them. Table~\ref{tab:nl} reports token-overlap $F_1$
($n=50$ per row). We report coverage here, as on the synthetic families, but not the
per-fragment restatement match: on free-text paragraphs it is not a clean recovery
diagnostic, since a non-empty restatement is not evidence that the answer-bearing
fact was recovered, so on natural language the end-to-end $F_1$ and coverage are
what we measure.

\begin{table}[t]
\centering
\small
\setlength{\tabcolsep}{4pt}
\begin{tabular}{llccccc}
\toprule
dataset & store & masked read & ceiling & colocated & no-protocol & cover/junk \\
\midrule
MuSiQue 2-hop & clean ($d{=}0$)        & $0.31$ & $0.58$ & $0.41$ & $0.52$ & $0.75$\,/\,$0$ \\
              & distractors ($d{=}8$)  & $0.44$ & $0.58$ & $0.03$ & $0.03$ & $1.00$\,/\,$8$ \\
\midrule
HotpotQA $k{=}2$ & clean ($d{=}0$)      & $0.58$ & $0.72$ & $0.65$ & $0.61$ & $0.90$\,/\,$0$ \\
                & distractors ($d{=}8$) & $0.57$ & $0.72$ & $0.03$ & $0.07$ & $1.00$\,/\,$8$ \\
\bottomrule
\end{tabular}
\caption{Real multi-hop text, query-blind ($n=50$ per row). Cells are $F_1$;
exact match follows the same pattern (MuSiQue $0.22$ then $0.34$; HotpotQA $0.42$
then $0.44$). Columns as in Table~\ref{tab:headline};
\emph{cover/junk} is coverage of the supporting paragraphs and the number of the
$d$ distractors the walk visited. The masked read uses a $64$-token restatement
budget; rerun at $256$ tokens the outputs are unchanged.}
\label{tab:nl}
\end{table}

The contamination result of Section~\ref{sec:eval-scale} holds on real text, and
against harder noise. With eight topical distractors the masked read holds,
$0.44$ on MuSiQue and $0.57$ on HotpotQA, while both unaddressed reads collapse to
near $0.03$: real distractors destroy a colocated read and leave the addressed one
where it was. The mechanism is not the one from Section~\ref{sec:eval-scale}.
There the walk could refuse disjoint junk; here the distractors share the query's
words, so the walk cannot avoid them and in fact visits all eight (junk-visited
$=8$, coverage $1.00$). The read survives anyway, because what protects it is the
mask, not the routing: each fragment is read in isolation, so a visited distractor
produces its own restatement but cannot corrupt a supporting fragment's masked
read, whereas colocation fuses them and the answer degrades. A visited distractor's
restatement does still enter the composition, so isolation bounds the damage rather
than removing it; that the read holds here means the reader can set the irrelevant
restatements aside, not that they were kept out. On real text the isolation does the
work that avoidance did on synthetic stores.

Two honesty notes. First, on a clean store the masked read does not win: with no
distractors to interfere, both unaddressed reads are as good or better (colocated
$0.41$ and no-protocol $0.52$ against $0.31$ on MuSiQue, $0.65$ and $0.61$ against
$0.58$ on HotpotQA), and the full-text ceiling is higher still. The read's value is robustness to contamination, not clean-store
accuracy, and the honest baseline is that colocation is fine until the store is
contaminated. Even so, these are reads of the store and not recalled answers: a
closed-book baseline, the same questions with no store at all, scores $F_1$ $0.07$
on MuSiQue and $0.10$ on HotpotQA, so the store's contents, not the question alone,
carry the answer. Second, coverage on the clean store is below one, $0.75$ on MuSiQue
and $0.90$ on HotpotQA: the lexical walk sometimes fails to reach the second hop
when the query and the bridging paragraph share no surface word. This lexical
dependence is the read's soft spot. It is also why the masked read rises from the
clean store to the contaminated one on MuSiQue, from $0.31$ to $0.44$: the fuller
walk over the larger store recovers second hops the clean walk missed (coverage
$0.75$ to $1.00$). A routing index that tolerates paraphrase rather than shared
tokens would remove the dependence; we return to it in
Section~\ref{sec:discussion}. The synthetic worst case makes the dependence
deterministic: on the DISCOUNT family (Appendix~\ref{app:precision}) one required
fragment shares no routing word with the query or the others, so the walk never
reaches it on any item, the masked read collapses to $0.01$, and a colocated read
that keeps that fragment in view scores $0.84$. The contamination-robustness of
Section~\ref{sec:eval-scale} is therefore conditional on routing recall, and the
store answers a query only when the query is connected by content to the fragments
the answer needs.

\subsection{It is addressing, not latent computation}
\label{sec:eval-deconfound}

One confound remains. The agents in this line reason in latent space before they
encode, so a masked read might not be addressing a stored fact at all; it might be
recovering something the agent \emph{computed} during that reasoning, in which case
the read would be entangled with latent computation and would fade if the agent did
less of it. We cut the confound by varying the agent's latent budget $\ell$, the
number of latent reasoning steps it takes while encoding a fragment, down to
$\ell=0$ where it does none and the cache holds only the encoded text. Reading off
masked-read accuracy across $\ell \in \{0, 20, 40\}$ (full grid, $n=100$ per cell,
in Appendix~\ref{app:ablation}), three of the four structures are flat: the chain
holds $0.93$ at every $\ell$, the pipeline and the symmetric store stay at
or above $0.99$. Removing latent computation entirely costs nothing. The read
recovers the encoded fact; it does not depend on the agent having reasoned over it,
which is what it means for the read to be addressing.

The hub is again the exception, rising with $\ell$ from $0.15$ to $0.40$. This is
not the addressing being computed: coverage and hub extraction are $1.00$ at every
$\ell$, so the walk reaches the fragment and the store returns its full contents
regardless of the budget. What moves with $\ell$ is downstream of the store's
output, in composing the recovered values into the answer, and we do not localize
it further; for the hub alone, then, the encoding agent's latent computation does
move the final answer, though it moves neither what is addressed nor what is
extracted. The deconfound also marks a claim we do
not make. Distilled latent state is why sharing caches is worthwhile
(Section~\ref{sec:intro}), since an agent can pass computed state that has no text
form, but the read does not require it: it fetches whatever the cache holds,
computed or merely encoded.

\subsection{The boundary: reading versus answering}
\label{sec:eval-boundary}

Every result so far has set the hub aside, and it is time to collect the debt.
The hub is where reading and answering come apart, and the gap has a precise
location. Table~\ref{tab:boundary} decomposes the hub read into a ladder, from
what the store returns to what the reader does with it.

\begin{table}[t]
\centering
\small
\begin{tabular}{lc}
\toprule
stage & score \\
\midrule
latent readout (coverage / restatement / hub extraction) & $1.00$ / $1.00$ / $1.00$ \\
decoded facts parseable                                   & $0.99$ \\
symbolic composition over the decoded facts               & $\mathbf{0.99}$ \\
\midrule
frozen reader composing the same facts                    & $\mathbf{0.40}$ \\
full-text ceiling (single context)                        & $0.94$ \\
\bottomrule
\end{tabular}
\caption{The hub, decomposed (the hub family of Table~\ref{tab:headline}, a single
query-blind draw, $n=100$). The store returns every fact (top block); a symbolic
composer, a deterministic rule over the decoded facts rather than the model,
answers at $0.99$, so the read's output determines the answer. The frozen reader
handed the same facts answers at $0.40$, below even its full-text ceiling of
$0.94$. Nearly all the residual is composition, in the reader.}
\label{tab:boundary}
\end{table}

The store does its part completely. Coverage, restatement, and hub extraction are
all $1.00$, and the decoded facts are parseable at $0.99$: the read returns every
value the answer needs, in usable form. The decisive line is the next one. A
symbolic composer, a fixed rule applied to those decoded facts instead of the
model, answers at $0.99$. Nothing the answer requires is missing from the read;
the read's output already determines the answer. Yet the frozen reader, handed the
same decoded facts, answers at $0.40$, below even the $0.94$ it reaches reading the
full text in one context. Nearly all of the residual is in the reader's
composition, not in the store or the routing; the symbolic composer is at $0.99$,
not $1.00$. And because the model reaches $0.94$ from the full text, part of what
it fails is composing from the read's terse restatements rather than composition in
the abstract.

The hub is where this shows because it concentrates many values in one fragment.
In the store of Table~\ref{tab:families}, a single fragment lists an inventory,
\emph{2 large red, 3 small green, 1 large blue}, and answering means looking each
item up in the two rate tables and summing, so the reader must select and combine
within one dense restatement, the operation it does worst; the chain, pipeline,
and symmetric stores put one value in each fragment and never trigger it. Nor is the
copy budget the constraint: read at a $256$-token budget the hub still answers
$0.36$, no better than at $64$, so the limit is composing the recovered values, not
fitting them into the restatement. We do not claim to close this gap, and locating it
is the point. It is exactly the boundary the thesis draws: a content-addressed
read makes the relevant facts reliably available, and turning them into an answer
is bounded by the reader. A stronger or task-tuned reader would move the $0.40$,
but that is a property of the reader; the read has already done its work, as the
symbolic composer's $0.99$ attests.

\section{Related work}
\label{sec:related}

For a broad map of latent-space methods spanning reasoning, memory, and
communication, see the recent survey of \citet{yu2026latentspacefoundationevolution};
we position here against the specific lines that bear on reading a merged store.

\paragraph{Latent multi-agent communication.}
A recent line has agents exchange hidden state rather than text. LatentMAS
\citep{zou2025latentmas} prepends one agent's KV cache to the next in sequence;
Agent Primitives \citep{jin2026agentprimitives} run parallel solvers whose caches a
selector consumes; Interlat \citep{interlat2025}, cache-to-cache transfer
\citep{fu2025cachetocache}, and thought communication \citep{zheng2025thoughtcomm}
pass latent messages between models, and RecursiveMAS \citep{yang2026recursivemas}
loops hidden states through a trained link. These papers study how to \emph{produce
and transfer} fragments; none, to our knowledge, defines how to \emph{read} an
unordered store of them for a query the producers never saw. The closest reader, the
Agent Primitives selector, sees every candidate at once, but its candidates are
redundant complete solutions to be voted among, not jointly necessary fragments to
be composed, so colocating them does not trigger the interference we study; the
analogue is close rather than exact.

\paragraph{KV caches: compression, reuse, and the addressing gap.}
Most KV-cache work optimizes either how much cache to keep, through eviction and
merging \citep{xiao2024streamingllm, zhang2023h2o, li2024snapkv, kvmerger2024}, or
how fast to reuse it, through block-structured serving and prefix caching
\citep{kwon2023pagedattention, gim2024promptcache, ma2024blockattention}; neither
asks which cache entry answers a given query. Closest to our setting is a
2025--2026 cluster that reuses \emph{independently encoded} chunks and reports the
failure we call colocation is not addressability: naively concatenating separately
computed caches breaks the cross-attention a joint prefill would have computed, at a
large accuracy cost \citep{cacheblend2025, ape2025, kvlink2025, epic2025}. Their
remedy is to \emph{repair} the union so the chunks read together again, by
recomputing a token subset \citep{cacheblend2025, epic2025}, aligning the attention
distribution \citep{ape2025}, or training link tokens \citep{kvlink2025}; ProphetKV
even steers the repair with the query \citep{prophetkv2026}. Like us, these systems
precompute each chunk's cache before any query, so encode-time query-blindness is
not our difference. The difference is at combine time: they assemble a per-query
retrieved set and repair its union with the query in hand, whereas we hold one
persistent, growing, query-independent store and never repair it, masking all but
the one addressed fragment so the interference cannot arise rather than being
recomputed away. These systems diagnose a related failure of joint reading from
independently encoded caches, which is evidence the phenomenon is real, but none of
them tests selective access to a designated fragment; addressing, not repair, is the
operation we add.

\paragraph{Attention as content-addressable memory.}
Treating a cache read as an address lookup has precedent. Modern Hopfield networks
show the transformer's own update rule is an associative-memory retrieval
\citep{ramsauer2021hopfield}, and that this retrieval degrades from clean
single-pattern recovery to a blurred average as the stored patterns grow numerous or
similar, which is the mechanistic reason a naive read over a contaminated pool is not
a clean lookup and the motivation for our hard mask. Memorizing Transformers
\citep{wu2022memorizing} make the read explicit, querying an external key-value
memory by approximate nearest neighbor, and Landmark Attention
\citep{mohtashami2023landmark} routes to a block through a learned per-block token,
the closest precedent for route-then-read. These retrievals are soft,
similarity-based, and learned within one model's own sequence; ours routes by lexical
overlap on an opaque tag signature, with optional Bloom acceleration \citep{bloom1970}, over
multi-party fragments merged after the fact. Parallel Context Windows \citep{pcw2023} is the
ancestor of masking-as-isolation, but it isolates every window uniformly so that a
fixed reader can attend to all of them, whereas we expose exactly one fragment by
content.

\paragraph{Learned latent memory.}
A parallel line builds latent memory that is written and later read, almost always
with a trained reader or writer: differentiable cache augmentation
\citep{liu2025dca}, uncertainty-triggered synthesis \citep{hou2026flashmem},
role-aware or self-generated agent memory \citep{fu2026latentmem, zhang2025memgen},
and reconstruction-oriented factual memory \citep{zhang2026nextmem}. These stores are
learned, and their memory is generated or consolidated rather than routed to. What
distinguishes ours is not single- versus multi-agent, since some of these are
multi-agent too, but that it is a replicated, multi-writer set merged by content
hash and training-free: fragments keep the exact caches their producers wrote, are
identified by content hash, and are read by a fixed signature match, with no learned
reader between the store and the answer.

\paragraph{Reading versus composing.}
Mechanistic studies of latent reasoning support our split between what the store
delivers and what the reader must do. Latent tokens transport and scaffold
computation rather than perform it \citep{latenttokens2026}, and multi-hop
composition is built by a model's own attention over a sequential chain rather than
by colocating independent vectors \citep{superposition2025}. Parallel Context Windows
report the same split from the systems side, helping comparison questions answerable
from independently read windows but hurting bridge questions that need one window
conditioned on another \citep{pcw2023}. This is why a merged store needs an active,
addressed read, and why, once the facts are recovered, composing them into an answer
remains the frozen reader's job, as the hub result and its widening on a weaker
reader (Appendix~\ref{app:llama}) show.

\paragraph{Replicated state and text retrieval.}
The merge our store inherits is a conflict-free replicated data type
\citep{shapiro2011crdt, preguica2018crdts}, established for these caches in the
companion work \citep{baquero2026cachemerging}; CRDT ideas have reached neural state
at the weight level \citep{gillespie2026crdtmerge}, but without any read or query
semantics. Finally, our read is not text retrieval. Retrieval-augmented generation
\citep{lewis2020rag} and its latent-space successors \citep{he2025clara} fetch text
or learned document representations selected for a known query; we address raw KV
fragments produced query-blind, and we set out the honest trade this makes against
text retrieval in Section~\ref{sec:disc-textret}.

\section{Limitations and discussion}
\label{sec:discussion}

The limitations follow the same routing--recovery--composition decomposition as
the task definition. First, routing is lexical. A query reaches a fragment
only through shared lexical terms (represented as equal tags), so a needed fragment that shares no routing token is
never read: on the DISCOUNT family this happens on every item, and a colocated read
that keeps everything in view does better there than the addressed read
(Appendix~\ref{app:precision}). On natural text the same dependence shows as a
coverage below one when a bridging paragraph shares no surface word with the query
(Section~\ref{sec:eval-nl}). A signature that matched by meaning rather than by
shared tokens would remove this, and a learned or embedding-based signature is the
natural next step; we use a lexical one here because it is deterministic, disclosed,
and needs no training, at the price of this recall boundary.

Second, isolated recovery is not purely latent. Writers and authorized query clients normalize
text before mapping its words to opaque tags, and the reader recovers each fragment
by decoding a short text restatement before composing. Keyless replicas see no plain
signature words, but the tags still expose equality, co-occurrence and access
patterns, so this is pseudonymized lexical addressing rather than private or
semantic search. This keeps the method training-free and deterministic, but it is a
step short of a read that stays in latent space from end to end. A signature and a
reader that operate on the cache without writer-side text or decoded restatements
are the direction we think most worth pursuing next. Addressing a fragment by its position in a canonical
order, rather than by its content, is a further complementary route we leave to
future work.

Third, final composition remains a reader property. On the hub a symbolic composer
over the read's output answers at $0.99$ while the frozen model answers at $0.40$
(Section~\ref{sec:eval-boundary}), and across model families the gap narrows as the
reader strengthens (Appendix~\ref{app:llama}). The store's part and the reader's
part come apart: MaSRead isolates and recovers a located fragment, but a stronger
or task-tuned reader is needed to move the second number.

The routed anomaly exercise exposes an additional upstream boundary. Its writer is
given an application threshold and emits a compact textual record before that
record is encoded as KV; it is neither unsupervised anomaly discovery nor a direct
latent-to-latent distillation protocol. The writer misses or misrecords two of 20
target anomalies and falsely labels two of 160 normal windows. MaSRead follows
those records rather than fixture truth, as a content-addressed store should. A
future system therefore needs to improve or verify payload formation separately
from improving the read.

Fourth, isolation is not constant-time search. Equation~\ref{eq:cost} makes the
distinction explicit: routing scans store metadata, $|V|$ may grow with lexical
connectivity, and every visited fragment is recovered and composed. Only the
materialized read of one already located block is independent of $|S|$. We report
decoded lengths and asymptotic structure rather than an end-to-end serving latency
benchmark.

Finally, the evidence has limits of scope. The powered results are on one model
family, Qwen at three scales, with a cross-family check on Llama; the
natural-language evidence is two multi-hop datasets at $k=2$; and several cells are
small ($n=20$ to $50$), with intervals in Appendix~\ref{app:stats}. The synthetic
families are lexically connected by construction, so their coverage of $1.00$ is in
part a property of the generator and not only of the router, which is why the
DISCOUNT counterexample and the coverage below one on natural text are the honest
tests of routing.

\subsection{Relation to text retrieval}
\label{sec:disc-textret}

The read we study is not text retrieval, and we do not benchmark against it. A
pipeline that stores and re-reads text is a different regime: it presupposes that
the source text is retained, which the setting here does not provide, and on
encode-only content it would likely win on verbatim accuracy and on storage. We
claim addressability, robustness to contamination, and read cost, not superiority
over text on those axes; the point of the setting is that agents may share distilled
reasoning rather than raw text, in which case there is no text to retrieve and the
question is only whether the latent payload can be read at all. A matched
comparison is therefore needed to establish when a latent payload is preferable:
it should hold source information and model capability fixed and measure storage,
write cost, routing cost, recovery and final-answer accuracy, and end-to-end
latency. We leave that regime-selection study to future work rather than infer it
from the present experiments.

\section{Conclusion}
\label{sec:conclusion}

When independent agents reason in latent space, their key-value caches accumulate
into a replicated store that a conflict-free merge keeps convergent under any order
or multiplicity of delivery. That store holds what the agents worked out, but
reading it whole does not recover it: the fragments interfere, and colocation is not
addressability. We repair the read by addressing it to content. A masked signature
read routes to a fragment by an opaque keyed lexical signature and decodes it under a hard
attention mask that hides the rest. When the required evidence is lexically
connected to the query, a walk over the signature graph reaches its fragments. On
our chain, pipeline, symmetric, hub, and natural-language stores, visited fragments
are recovered in isolation; the recovery remains effective as unrelated fragments
accumulate and transfers to a different model family. Once a fragment is located,
its materialized read cost does not grow with the store, although routing and the
number of reads remain store dependent.

We have been equally explicit about where the method stops. The store delivers the
facts a query needs when routing reaches them; composing them into an answer is bounded by the reader, and we
locate that limit rather than hide it. The routing is lexical beneath its opaque tags, so a fragment a query
cannot reach by shared lexical terms is a fragment the read cannot use. And the read is not
yet purely latent: text is normalized by authorized endpoints for routing and is
decoded again for recovery. These boundaries define concrete directions for
future work. Most consequential is the last: a signature and reader that operate
on the cache without decoding it to words would make the read latent from end to
end.

The claim we make is narrow. Once agents reason in latent space, the store they
produce can be read for queries they never saw, when those queries are connected by
content to the fragments they need, provided the store is addressed rather than
merely colocated. Addressing, not merging, is the operation that makes a replicated
latent store usable, and it is the operation this paper supplies.

\appendix

\section{Anomaly probes: retention and routed extraction}
\label{app:anomaly}

\subsection{Retention and salience in one undistilled cache}

A small probe supporting the motivation of Section~\ref{sec:intro}: a single
agent (Qwen3-1.7B) reads a short event stream \emph{query-blind}. The stream is
a ten-entry temperature log, every reading near $20^{\circ}$C except one at
$87.6^{\circ}$C. The log precedes any question in the sequence, so by causal
attention its KV cache does not depend on the question; a matched control
replaces the anomaly with $20.5$. Decoding is greedy.

\textbf{Behavioral.} Reading the query-blind cache, the agent \emph{names the
anomaly on a generic prompt}: asked only ``is there anything notable?'' it
answers ``a notable spike in temperature at 08:25, reaching $87.6^{\circ}$C,
significantly higher than the surrounding readings.'' Asked which reading is
anomalous it returns $87.6$ exactly. On the control it reports ``no significant
anomalies or extreme values'' and never emits $87.6$.

\textbf{Cache inspection.} We measure the attention each reading receives from
the question tokens, averaged over all layers and heads (a content-to-content
comparison, which avoids the attention-sink artifact of absolute attention
mass). The anomalous reading is the most attended of the ten
(Table~\ref{tab:anomaly}), receiving about $1.9\times$ the mean regular reading.

\begin{table}[ht]
\centering
\small
\begin{tabular}{lc}
\toprule
reading & attention received \\
\midrule
\textbf{87.6 (anomaly)} & \textbf{0.0032}\ \ (rank 1/10) \\
20.3 & 0.0027 \\
19.9 & 0.0015 \\
20.2 & 0.0010 \\
\bottomrule
\end{tabular}
\caption{Attention each reading receives from the question tokens (mean over all
layers and heads; representative rows). The anomaly is rank $1/10$.}
\label{tab:anomaly}
\end{table}

\textbf{Scope.} This first probe is illustrative (one model, one toy stream): it shows the
anomaly is \emph{retained, recoverable, and salient} under query-blind encoding.
It does not demonstrate a compact computed payload. Reproduce with
\texttt{paperJMLR/probes/anomaly\_probe.py}. The probe writes only to standard
output, and no historical capture was retained. The accompanying capture is an
authorized greedy regeneration of this fixed fixture; its checksum, Qwen snapshot,
and software environment are recorded in the results manifest.

\subsection{Full-system routed fixed-rule extraction}

We next test a compact computed payload without changing the MaSRead mechanism.
Each fixture contains one ten-reading temperature window with exactly one value
above $40.0^{\circ}$C, a matched all-normal window, and up to seven unrelated
all-normal windows of the same textual form. Before any read query is known, a
Qwen3-1.7B writer receives one local window and the fixed application rule
``anomaly iff one reading exceeds $40.0^{\circ}$C.'' It greedily emits exactly
three fields: status, timestamp, and one-decimal value. A strict parser rejects
malformed or incoherent output; it neither checks against fixture truth nor
repairs or retries a result.

For an anomaly record, the system adds a fixed controlled descriptor set
\{\texttt{anomaly}, \texttt{anomalous}, \texttt{outlier}, \texttt{spike},
\texttt{notable}\}. It then query-blind encodes the compact record as a KV
fragment and constructs the ordinary MaSRead sidecar. As in
Section~\ref{sec:masread}, alphabetic content words are mapped to opaque
HMAC--SHA256 tags; neither the measurement nor timestamp contributes a signature
word. The authorized query client maps its words under the same key. Thus a query
such as ``report the outlier reading'' can route through \texttt{outlier} and only
then recover the numerical value from the masked payload. A raw log has no usable
signature after same-format boilerplate removal, and a timestamp-only query has no
alphabetic bridge.

We developed this version after preserving an unsuccessful preliminary writer
smoke test that used an underspecified ``obvious outlier'' instruction. We froze
the threshold, prompt, strict parser, and direct non-thinking chat protocol after
a separate seed-314159 writer gate and one disjoint full-path smoke. The evaluation
then ran once on seed 42, items 1--20; the preliminary test's observed item 0 was
excluded prospectively. For each item we evaluate four fixed query wordings with
$D\in\{1,3,7\}$ routine windows, giving 240 end-to-end records and 180 distinct
one-shot writer records. There are no output retries, exclusions, or gold repairs.

\begin{table}[ht]
\centering
\small
\begin{tabular}{@{}lr@{}}
\toprule
stage or control & frozen result \\
\midrule
writer format valid & $180/180$ \\
writer anomaly status / exact value & $19/20$ / $18/20$ \\
writer normal status & $158/160$ \\
target seeded / target first & $224/240$ / $216/240$ \\
gold value in masked target restatement & $216/240$ \\
gold value in final facts-only answer & $208/240$ \\
raw-log walk empty & $240/240$ \\
all-normal walk empty & $224/240$ \\
actual timestamp-only walk empty & $240/240$ \\
\bottomrule
\end{tabular}
\caption{Routed fixed-rule anomaly extraction. Counts separate writer detection,
opaque-tag routing, masked recovery, and final composition; errors are retained.}
\label{tab:routed-anomaly}
\end{table}

The decomposition is exact. One target is labeled normal and never routes; a
second target is labeled anomalous but with the wrong timestamp and value. Those
two writer records account for every masked gold-value failure. Two normal windows
are falsely labeled anomalous, producing the only routine visits and all-normal
control failures. For the 18 correct target records, masked timestamp and value
recovery is exact for every query and $D$. Final composition then preserves the
value in $208/216$ such cases: three errors use a contaminated fact set created by
a writer false positive, and five are truncated explanatory answers under the
frozen 32-token budget. We observed no independent HMAC or selector failure.

This is stronger than the retention probe because it exercises local finding
formation, compact sharing, opaque routing, masked recovery, and composition in
one system. Its scope remains precise: it is threshold-rule extraction through a
textual compact-record interface, not general anomaly discovery and not an agent
emitting a directly distilled latent without an intermediate record. Reproduce
with \texttt{paperJMLR/probes/routed\_anomaly\_probe.py}; the frozen artifacts,
protocol hash, environment, and complete failure audit are listed in the results
manifest.

\section{Scale: the read at 4B and 8B}
\label{app:scale}

The mechanism panel of Section~\ref{sec:masread} is measured at 1.7B. Here we
repeat it at 4B and 8B (Qwen3-4B and Qwen3-8B), under the same query-blind
two-fragment encodes and the same three reads. Table~\ref{tab:scale-panel} gives
all three scales.

\begin{table}[ht]
\centering
\small
\begin{tabular}{lccc}
\toprule
read & 1.7B & 4B & 8B \\
\midrule
open (colocated)      & $0.625$ & $0.525$ & $0.525$ \\
masked (target block) & $1.000$ & $0.750$ & $0.875$ \\
wrong-mask control    & $0.000$ & $0.050$ & $0.000$ \\
\bottomrule
\end{tabular}
\caption{Mechanism panel across scale (query-blind, $n=40$ pairs). At every scale
the masked read sits far above both the open read and the wrong-mask control, so
the mask, not the query, governs which fragment is read. The masked row is
$1.000$, $0.750$, $0.875$: the 4B value is a non-monotonic dip, not a trend.}
\label{tab:scale-panel}
\end{table}

The addressing result holds at every scale: the masked read is far above the open
read and the wrong-mask control throughout. What changes with scale is only how
completely masking removes the interference, and it does not change monotonically.
At 1.7B masking dissolves it ($1.000$). At 4B it suppresses rather than dissolves:
the masked read ($0.750$) sits well above the open read but below the isolated
single-fragment read, which is at ceiling there ($20/20$). At 8B the masked read
recovers to $0.875$, at or above the isolated read, which is itself off ceiling at
this scale under query-blind decoding ($14/20 = 0.70$). The 4B dip is therefore
non-monotonic, and the qualitative result, that the mask decides what is readable,
is scale-invariant.

The end-to-end read holds at 8B as well. Table~\ref{tab:eightb} runs the full
protocol on the four families at 8B ($k=3$, $n=50$).

\begin{table}[ht]
\centering
\small
\begin{tabular}{lcccc}
\toprule
family & masked read & colocated & ceiling & control \\
\midrule
chain                  & $1.00$ & $0.90$ & $1.00$ & $0.02$ \\
pipeline               & $0.98$ & $0.32$ & $1.00$ & $0.02$ \\
symmetric              & $0.78$ & $0.18$ & $1.00$ & $0.00$ \\
hub                    & $0.62$ & $0.32$ & $0.88$ & $0.46$ \\
\bottomrule
\end{tabular}
\caption{End-to-end read at 8B (Qwen3-8B; $k=3$; query-blind; $n=50$). The masked
read stays well above the colocated read on every family. The no-protocol control
is at floor except on the hub, where the larger model answers the hub question from
the merged cache without the read's facts at $0.46$, above the colocated read there;
so the ordering masked above colocated above control holds on the chains and the
symmetric store but not on the hub.}
\label{tab:eightb}
\end{table}

The pattern of Section~\ref{sec:eval} survives the change of scale: the masked read
is at or near ceiling on the chains, the colocated read collapses, and the control
is at floor. The two family-level shifts are honest to report. The hub reads higher
at 8B ($0.62$) than at 1.7B, but the no-protocol control on the hub also rises to
$0.46$, the larger model answering it from the merged cache without the read's facts,
so the hub's margin over the control narrows rather than widens. The symmetric family reads lower at 8B ($0.78$) than at 1.7B; the depth
behavior at 8B is in Appendix~\ref{app:depth}.

\section{Cross-family replication (Llama)}
\label{app:llama}

The main results use the Qwen family at three scales (Appendix~\ref{app:scale}). To
test whether the read is specific to that family, we repeat the end-to-end read on
Llama at two capacities, Llama-3.2-3B and Llama-3.1-8B
\citep{llama3herd, llama32}, on the chain, pipeline, and symmetric families,
query-blind, $n=50$ per cell. The recovered launchers record the exact Hugging
Face repositories \path{unsloth/Llama-3.2-3B-Instruct} and
\path{unsloth/Meta-Llama-3.1-8B-Instruct}, both with data seed $42$. Read-only
inspection of the execution host's Hugging Face cache identifies one snapshot per
repository, referenced by \texttt{main} and timestamped before its run. The pinned
revisions are:
\begin{center}
\small
\begin{tabular}{@{}ll@{}}
3B & \texttt{006f5dcd1393c3add266de40994ba96225e9689d} \\
8B & \texttt{a2856192dd7c25b842431f39c179a6c2c2f627d1}
\end{tabular}
\end{center}
Table~\ref{tab:llama} adds a
\emph{restate} column, the per-fragment restatement exact match, which separates
what the store returns from what the reader does with it.

\begin{table}[ht]
\centering
\small
\begin{tabular}{llcccccc}
\toprule
model & family & masked & ceiling & colocated & control & coverage & restate \\
\midrule
Llama-3.2-3B & chain     & $0.46$ & $0.86$ & $0.00$ & $0.06$ & $1.00$ & $0.89$ \\
             & pipeline  & $0.24$ & $0.98$ & $0.02$ & $0.00$ & $1.00$ & $0.93$ \\
             & symmetric & $0.50$ & $0.98$ & $0.02$ & $0.00$ & $1.00$ & $0.98$ \\
\midrule
Llama-3.1-8B & chain     & $0.96$ & $0.88$ & $0.18$ & $0.20$ & $1.00$ & $1.00$ \\
             & pipeline  & $0.72$ & $1.00$ & $0.04$ & $0.00$ & $1.00$ & $1.00$ \\
             & symmetric & $0.56$ & $0.90$ & $0.20$ & $0.04$ & $1.00$ & $0.99$ \\
\bottomrule
\end{tabular}
\caption{Cross-family read on Llama at two capacities (query-blind, $n=50$ per
cell). Columns are as in Table~\ref{tab:headline}, plus \emph{restate}, the
per-fragment restatement exact match. Coverage and restatement are high throughout,
so routing and extraction transfer off the Qwen family; the masked read rises
sharply from 3B to 8B, to the full-text reference on the chain though below it on
the pipeline and symmetric families, where the residual is composition
(Section~\ref{sec:eval-boundary}).}
\label{tab:llama}
\end{table}

Three things transfer off the Qwen family. Coverage is $1.00$ on every cell, so the
signature walk reaches every required fragment on Llama as on Qwen. The colocated
read collapses, to at most $0.20$, while the masked read stays well above it, the
interference of Section~\ref{sec:failure} again. And restatement is faithful, from
$0.89$ at 3B to $1.00$ at 8B, so the masked decode extracts the fragments' contents
rather than only routing to them: the store delivers the facts off-family.

What does not transfer as a constant is the absolute read accuracy, and it tracks
the reader's capability rather than the model family. The masked read rises sharply
from 3B to 8B, from $0.24$--$0.50$ to $0.56$--$0.96$, matching the full-text
reference on the chain (the $0.96$ against $0.88$ difference is within sampling
noise at $n=50$) though staying below it on the pipeline and symmetric families.
That the variable is reader quality and not parameter count is clear from
Qwen3-1.7B, which reads at $0.97$ where the larger Llama-3.2-3B reads at $0.24$; the
two Llama models differ in both size and generation, so we do not separate scale
from architecture, and claim only that a more capable reader reads the store
better.

Where the 8B read still sits below the reference, on the pipeline ($0.72$ against
$1.00$) and the symmetric store ($0.56$ against $0.90$), extraction is not the
obvious cause: restatement there is exact ($1.00$ and $0.99$). That match counts
numeric recovery, not whether every relation and binding was preserved, so unlike
the hub's symbolic-composer test (Section~\ref{sec:eval-boundary}) it shows the
residual is consistent with composition rather than fully localizing it there. The
residual moves with the reader, widest at 3B, narrower at 8B, and closed on the
chain; the cross-family evidence is that the store and its addressed read transfer,
and the answer quality that remains is bounded by the reader.

\section{Depth: \texorpdfstring{$k{=}5$}{k=5} stores}
\label{app:depth}

The main results use three-fragment stores. This appendix pushes the read to
five, on two families that extend to depth: an additive chain in which every hop
is jointly necessary, and a five-way hub.

\begin{table}[ht]
\centering
\small
\begin{tabular}{lccccc}
\toprule
family & query-anticipated & masked read & ceiling & colocated & coverage \\
\midrule
additive chain (5 hops) & $0.60$ & $1.00$ & $1.00$ & $0.65$ & $1.00$ \\
hub ($k{=}5$)           & $0.05$ & $1.00$ & $1.00$ & $0.75$ & $1.00$ \\
\bottomrule
\end{tabular}
\caption{Five-fragment stores at 1.7B ($n=20$ per cell). \emph{query-anticipated}
is a superseded reference in which the query was present while the fragments were
encoded; \emph{masked read} is the canonical query-blind read. The query-blind read
reaches the full-text ceiling on both families with coverage $1.00$.}
\label{tab:k5}
\end{table}

Depth is not the obstacle it first appears. Read with the query present in the
encode, the five-fragment stores look hard, $0.60$ on the chain and $0.05$ on the
hub. That difficulty is an artifact of query-anticipated encoding: the masked read
of a fragment encoded alongside the query restates the embedded query text rather
than the fragment's content (the sentence-versus-copy control of
Appendix~\ref{app:signatures}). Under the canonical query-blind encode, which is
also the reusable-memory regime the paper targets, both families reach the ceiling
with full coverage, so five-fragment stores need no new mechanism. The colocated
read again sits below the masked read ($0.65$ and $0.75$), the interference of
Section~\ref{sec:failure} at depth. One caveat: this $k=5$ hub is a different, and
easier, generator than the $k=3$ hub of Section~\ref{sec:eval-boundary} (its
ceiling is $1.00$, not $0.92$), so the depth result does not speak to that
family's compose boundary.

The read holds up a depth ramp at 8B as well. Table~\ref{tab:depth8b} runs $k=3$
through $k=5$ on the additive chain and the hub at 8B.

\begin{table}[ht]
\centering
\small
\begin{tabular}{llcccc}
\toprule
$k$ & family & masked read & colocated & ceiling & control \\
\midrule
$3$ & additive chain & $0.95$ & $0.45$ & $1.00$ & $0.00$ \\
$3$ & hub            & $0.95$ & $0.25$ & $1.00$ & $0.00$ \\
$4$ & additive chain & $0.80$ & $0.00$ & $1.00$ & $0.00$ \\
$4$ & hub            & $0.90$ & $0.00$ & $1.00$ & $0.00$ \\
$5$ & additive chain & $0.85$ & $0.00$ & $1.00$ & $0.00$ \\
$5$ & hub            & $0.95$ & $0.00$ & $1.00$ & $0.00$ \\
\bottomrule
\end{tabular}
\caption{Depth ramp at 8B (Qwen3-8B; query-blind; $n=20$ per cell). Across $k=3$ to
$k=5$ the masked read stays high while the colocated read collapses to zero and the
control is at floor. $k=2$ is the base case of the mechanism panel
(Appendix~\ref{app:scale}); the additive-chain generator is degenerate at two
fragments and is omitted here. The hub here is the easier depth-scalable generator
of the $k=5$ table above, not the hub of Table~\ref{tab:eightb} and
Section~\ref{sec:eval-boundary}, so its accuracy is not comparable to those.}
\label{tab:depth8b}
\end{table}

At 8B the masked read stays between $0.80$ and $0.95$ as the store grows from three
to five fragments, while the colocated read falls to zero by $k=4$ and the control
stays at floor. Deeper stores cost the addressed read a little accuracy but do not
break it, and the gap over the colocated read widens with depth rather than
closing. Figure~\ref{fig:depth8b} plots the same reads across the ramp.

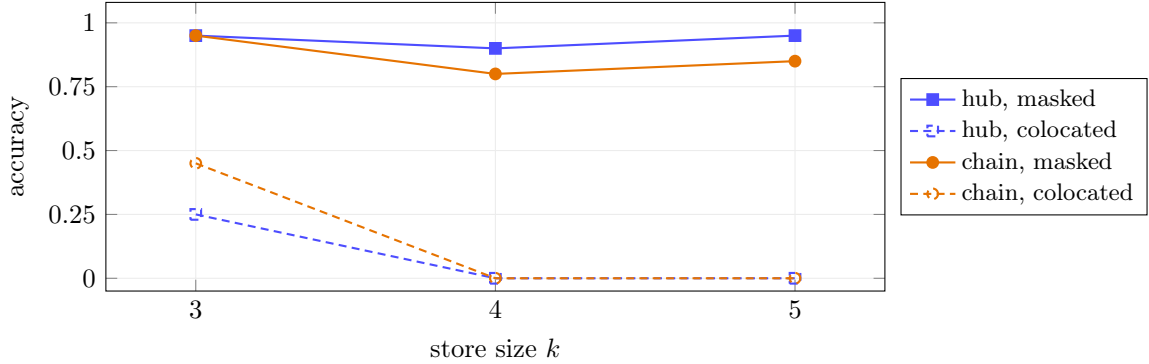
\begin{figure}[ht]
\centering
\begin{tikzpicture}
\begin{axis}[
  width=0.78\textwidth, height=5.4cm,
  xlabel={store size $k$}, ylabel={accuracy},
  xmin=2.7, xmax=5.3, xtick={3,4,5},
  ymin=-0.05, ymax=1.08, ytick={0,0.25,0.5,0.75,1.0},
  legend style={at={(1.02,0.5)}, anchor=west, font=\footnotesize},
  legend cell align={left},
  grid=major, grid style={gray!15},
  tick label style={font=\small}, label style={font=\small},
]
\addplot[blue!70, thick, mark=square*]
  coordinates {(3,0.95)(4,0.90)(5,0.95)};
\addlegendentry{hub, masked}
\addplot[blue!70, thick, densely dashed, mark=square]
  coordinates {(3,0.25)(4,0.00)(5,0.00)};
\addlegendentry{hub, colocated}
\addplot[orange!90!black, thick, mark=*]
  coordinates {(3,0.95)(4,0.80)(5,0.85)};
\addlegendentry{chain, masked}
\addplot[orange!90!black, thick, densely dashed, mark=o]
  coordinates {(3,0.45)(4,0.00)(5,0.00)};
\addlegendentry{chain, colocated}
\end{axis}
\end{tikzpicture}
\caption{The 8B depth ramp from $k=3$ to $k=5$ (query-blind, $n=20$ per point; the
points of Table~\ref{tab:depth8b}). As the store grows the colocated read collapses
on both families while the masked read stays high.}
\label{fig:depth8b}
\end{figure}

\section{Latent-computation ablation: full grid}
\label{app:ablation}

Section~\ref{sec:eval-deconfound} cuts the latent-computation confound by varying
the encoding agent's latent budget $\ell$, the number of latent reasoning steps it
takes before it encodes a fragment. At $\ell=0$ the agent does none and the cache
holds only the encoded text. Table~\ref{tab:ablation} gives the full grid: the
masked read, query-blind, seed $42$, $n=100$ per cell, under the canonical read
configuration of Appendix~\ref{app:repro}.

\begin{table}[ht]
\centering
\small
\begin{tabular}{lccc}
\toprule
family & $\ell{=}0$ & $\ell{=}20$ & $\ell{=}40$ \\
\midrule
chain                  & $0.93$ & $0.93$ & $0.93$ \\
pipeline               & $0.99$ & $1.00$ & $1.00$ \\
symmetric              & $1.00$ & $1.00$ & $1.00$ \\
hub                    & $0.15$ & $0.23$ & $0.40$ \\
\midrule
\multicolumn{4}{@{}l}{\emph{hub diagnostics, per $\ell$:}} \\
\quad coverage               & $1.00$ & $1.00$ & $1.00$ \\
\quad hub extraction (exact) & $1.00$ & $1.00$ & $1.00$ \\
\bottomrule
\end{tabular}
\caption{Masked-read accuracy as the encoding agent's latent budget $\ell$ varies
(query-blind, seed $42$, $n=100$ per cell). Three families are flat: removing
latent computation entirely costs nothing, so the read addresses the encoded fact
rather than exploiting computation over it. The hub rises with $\ell$, but its
coverage and its exact hub extraction are $1.00$ at every budget, so the store
returns every value regardless of $\ell$; what the extra steps move is downstream
of that output, in composing the recovered values into the answer, not the
addressing.}
\label{tab:ablation}
\end{table}

The three flat rows are the deconfound. Their read is unchanged from $\ell=40$ down
to $\ell=0$, where the agent contributes no computation at all, so the masked read
recovers what was encoded rather than what was reasoned. The hub is the one family
whose answer moves with $\ell$, and the diagnostic rows localize the movement: at
every budget the walk reaches the hub fragment (coverage $1.00$) and the store
returns its full contents (exact extraction $1.00$), so the values the answer needs
are present regardless of $\ell$. What the latent steps move is downstream of that
output, in composing the recovered values into the answer, a property of the reader
(Section~\ref{sec:eval-boundary}) rather than of the addressing; we do not localize
it further.

\section{Signature construction and negative results}
\label{app:signatures}

\paragraph{Construction.}
A writer extracts its normalized word set locally, then maps every word to a
$128$-bit tag using HMAC-SHA256 with domain separator
\texttt{masread/signature/v1} and the store-epoch key. The public key identifier
is the first $128$ bits of SHA-256 of the key; the key itself is never serialized.
The immutable sidecar contains the raw enumerable tag set, its schema and key
identifiers, and a digest binding those fields. A payload identifier in turn binds
the historical content hash to that sidecar digest. Replicas reject mixed
schema/key epochs, an invalid sidecar digest, or an invalid payload identity. Key
rotation therefore creates a new epoch and requires retagging.

The set-relative boilerplate intersection is computed after merge in tag space,
not by the writer: adding an element can shrink that intersection and restore tags
to existing final signatures. A $128$-bit Bloom filter over each final tag set can
accelerate membership tests (\emph{does this fragment contain tag $t$}), but the
enumerable exact set grows the frontier; Algorithm~\ref{alg:masread} never attempts
to enumerate Bloom bits. Matching and frontier growth are set operations and are
order-independent, while the visit order uses the existing content-derived
tie-break. The reported runs use exact tag overlap (\texttt{token\_overlap}); a
model-free audit over all available evaluation items verifies that substituting
tags for normalized words leaves every seed score, visited set, and visit order
unchanged. Thus no model result is regenerated for this representation change.

For the short synthetic signatures a $128$-bit Bloom filter has a negligible
false-positive rate; on long natural-language paragraphs it is more saturated.
A false positive over-visits: the extra fragment is isolated under its own mask,
but its tags then enter the frontier and can draw in further fragments, so the
cost can cascade rather than stopping at one visit.

\paragraph{A text-free selector does not work.}
We asked whether a fragment could be addressed with no writer-side text at all,
from its key-value keys alone: a SimHash sketch over the fragment's key tensors at
a middle layer. If it worked, selection would live in the same key space as the
interference of Section~\ref{sec:failure}. It does not.
Table~\ref{tab:keysketch} compares it against the lexical signature ($n=40$
pairs).

\begin{table}[ht]
\centering
\small
\begin{tabular}{lcccc}
\toprule
 & \multicolumn{2}{c}{block selected (top-1)} & \multicolumn{2}{c}{masked read} \\
\cmidrule(lr){2-3}\cmidrule(lr){4-5}
selector & template & shared-unit & template & shared-unit \\
\midrule
lexical signature       & $1.00$ & $1.00$ & $1.00$ & $0.98$ \\
key-sketch (KV keys)    & $0.78$ & $0.45$ & $0.82$ & $0.42$ \\
\bottomrule
\end{tabular}
\caption{A text-free selector built from the fragment's key-value keys, against
the lexical signature ($n=40$ pairs; sketch at layer $15$, best over bit-budgets
$\{16,32,64\}$). The key-sketch selects the right block only $0.45$ to $0.78$ of
the time, and the masked read on its choice trails accordingly, where the lexical
signature selects and reads at ceiling. Addressing needs a lexical handle; opaque
tags avoid disclosing the words themselves to keyless replicas while retaining
equality and access-pattern leakage (Section~\ref{sec:disc-textret}).}
\label{tab:keysketch}
\end{table}

\paragraph{The query-echo artifact.}
One further negative explains a difficulty that is not one. When a fragment is
encoded with the query already present, its block contains the query text, and a
masked read can restate that embedded query rather than the fragment's own content.
This is why five-fragment stores looked hard under query-anticipated encoding
($0.60$ on the chain, $0.05$ on the hub) and reached the ceiling once the query was
removed (Appendix~\ref{app:depth}). A sentence-versus-copy control localizes it:
comparing a free sentence restatement, which can echo the embedded query, against a
copy-prompted restatement of the block attributes the gap to the echo rather than to
a read deficit or to depth. The paper's canonical configuration encodes query-blind
and restates by copy prompt, so the artifact does not arise in any reported number.

\section{Determinism and CRDT properties}
\label{app:determinism}

This appendix states the convergence property behind Section~\ref{sec:store} and
the materialized read behind the cost claim of Section~\ref{sec:eval-scale}.

\paragraph{Convergence.}
The store is a grow-only set of fragments keyed by content hash, and the merge is
set union. Union is commutative, associative, and idempotent, so replicas that
have received the same fragments hold the same set regardless of the order or the
multiplicity in which the fragments arrived, and converge once delivery is
eventual; the agreement is exact up to hash collisions. The consequence used
throughout the paper is stronger than set agreement: the render of a set of
fragments is \emph{byte-identical} under any permutation or duplication of the
inputs, a determined RoPE layout rather than an order-dependent concatenation.
That property, with its proof and its verification on the real key-value state of
Qwen3-1.7B ($28$ layers) and Qwen3-4B ($36$ layers), is established in the
companion work on cache merging \citep{baquero2026cachemerging}; it is what makes
the order and duplication robustness of this paper structural rather than
something the read has to tolerate.

\paragraph{The materialized read.}
The canonical layout places each fragment's block at a determined position in the
render: the first block keeps the native coordinates $[0,\text{len})$ it was
encoded at, and each later block is rotated to begin where the previous one ends.
The offset is exactly what a query-blind position cannot address
(Section~\ref{sec:failure}), since the same fragment lands at a different offset in
a different store. An in-place masked read leaves the block at its offset and
decodes over the whole render, so its cost grows with the store. A
\emph{materialized} read instead extracts the block and reverses its rotation by
that offset, restoring the native coordinates near $0$, and reads it standalone.
Table~\ref{tab:rebased} runs both reads up a $k$-ramp.

\begin{table}[ht]
\centering
\small
\begin{tabular}{lcccc}
\toprule
$k$ & render length & colocated & masked (in place) & masked (materialized) \\
\midrule
$2$  & $292$  & $0.60$ & $1.00$ & $1.00$ \\
$4$  & $585$  & $0.05$ & $1.00$ & $1.00$ \\
$8$  & $1172$ & $0.00$ & $1.00$ & $1.00$ \\
$16$ & $2343$ & $0.00$ & $1.00$ & $1.00$ \\
\bottomrule
\end{tabular}
\caption{The read up a $k$-ramp ($n=20$ per cell). The render grows roughly
linearly with the store, from $292$ to $2343$ tokens, and the colocated read
collapses. Both masked reads stay at $1.00$; the materialized read attains it while
decoding a fixed $147$-token block at every $k$, so its cost is independent of the
store size, which is the read-cost claim of Section~\ref{sec:eval-scale}. Both reads
are at $1.00$ over this range, so we do not observe the in-place read degrading at
larger offsets; materializing removes any such dependence by construction, since it
reads at native coordinates regardless of the block's offset in the render.}
\label{tab:rebased}
\end{table}

Both reads recover the target at every store size, so the materialized view is not
a different mechanism but the same read relocated to fixed coordinates. Its value
is the cost: where the in-place read attends over a render that grows with every
fragment merged, the materialized read decodes a block whose length does not depend
on how large the store has become.

\section{The precision boundary}
\label{app:precision}

Both the failure the read repairs and the failure the read itself has are set by
one thing: how lexically similar the fragments are.

\paragraph{Interference is a lexical effect.}
The colocation interference of Section~\ref{sec:failure} is not uniform. We read one
fragment beside a single partner and vary the partner (Table~\ref{tab:ladder}). A
partner from a different family leaves the read almost untouched ($0.97$); a partner
that shares the template, or even one unit, drives it down to about $0.4$, and most
of the errors are the partner's value substituted for the target's. Interference,
and therefore the need to address rather than colocate, arises precisely when
fragments are lexically similar, as the evaluation families are by construction.

\begin{table}[ht]
\centering
\small
\begin{tabular}{lcc}
\toprule
partner beside the read fragment & $n$ & accuracy \\
\midrule
none (isolated)               & $100$ & $1.00$ \\
different family              & $100$ & $0.97$ \\
same template, disjoint units & $200$ & $0.40$ \\
same template, one shared unit & $148$ & $0.44$ \\
\bottomrule
\end{tabular}
\caption{Interference as a function of partner similarity (query-blind read of one
fragment beside one partner). A dissimilar partner barely affects the read; a
lexically similar one collapses it, mostly by substituting its own value.}
\label{tab:ladder}
\end{table}

\paragraph{Routing recall is also lexical: the \textsc{discount} boundary.}
The same dependence bounds the read from the other side. The walk routes by shared
words, so a required fragment that shares no routing token with the query or with an
already-visited fragment is never reached. The \textsc{discount} family is built to
contain exactly one such fragment: a store of a per-unit price, a group-discount
rule, and a party composition, asked for the total. The question seeds the walk on
the discount-rule and party fragments, which echo its words, but the price fragment
shares no routing token and is omitted on every one of the $100$ items, so coverage
is a fixed $2/3$. Because the price is needed, the masked read collapses to $0.01$
against a full-text ceiling of $1.00$. Here, unusually, the colocated read beats it,
at $0.84$: keeping every fragment in scope preserves the price the walk cannot
reach, so the addressed read's own selectivity is the liability. This is the
deterministic worst case of the lexical-mismatch soft spot of
Section~\ref{sec:eval-nl}: a routing index tolerant of paraphrase or meaning, rather
than of shared tokens, would close it, and lexical routing cannot.

\section{Statistical procedures}
\label{app:stats}

\paragraph{Metrics.}
On the synthetic families the answer is an integer, and a cell's accuracy is the
fraction of items whose parsed, normalized answer equals the gold. On the
natural-language sets the answer is a free-form string, scored by the standard
HotpotQA token-level $F_1$ and exact match; a cell reports the mean $F_1$ over its
items. Coverage is the fraction of a query's required fragments the walk reaches,
and hub extraction is the fraction of hub items whose full contents the read
returns exactly.

\paragraph{Intervals.}
Where an interval is shown, for example the bracketed intervals in the headline of
Table~\ref{tab:headline}, it is the $95\%$ Wilson score interval for a binomial
proportion. We use the Wilson interval rather than the normal (Wald) approximation
because several cells sit near $0$ or $1$, where Wald intervals are too narrow and
can leave the unit interval. Tables without an interval report a point estimate at
the sample size stated in their caption.

\paragraph{Sample sizes.}
Each table and figure states its own per-cell $n$. In summary: the headline is
$n=300$; the mechanism panel and the scale panel are $n=40$ pairs; the
contamination sweep and the natural-language rows are $n=50$; the latent ablation
and the read-versus-answer decomposition are $n=100$; and the five-fragment, depth,
and $k$-ramp tables are $n=20$ per cell.

\paragraph{Seeds and pooling.}
Decoding is greedy, so a run is deterministic given its data seed and the only
randomness is which problems are drawn. The headline of Table~\ref{tab:headline}
pools three independent data seeds ($42$, $7$, $100$), each drawing a fresh set of
$100$ problems per family; the reported accuracy is computed over the pooled $300$
items and the Wilson interval is taken on that pooled count, so we pool the draws
rather than averaging three per-seed accuracies. All other cells use the single
data seed $42$ at the sample size given in their caption.

\section{Datasets and reproducibility}
\label{app:repro}

This appendix gives the store constructions and the exact read configuration
behind every number in the paper.

\paragraph{Synthetic stores.}
All four families are generated by \texttt{data\_partitioned\_k3.py}, three
fragments each, with the numbers and content words drawn at random per item;
Table~\ref{tab:families} shows one instance of each. In every family the three
fragment texts are produced without the question, which is a separate field, so
the encode is query-blind. \emph{Chain} (unit conversion): the fragments state
$1\,\text{big}=K\,\text{mid}$, $1\,\text{mid}=J\,\text{small}$, and a quantity $N$
(a multiple of $KJ$) in the small unit; the answer is $N/(KJ)$. \emph{Pipeline}
(affine): two machines $x\mapsto P_1x+Q_1$ and $y\mapsto P_2y+Q_2$ and a starting
count $X_0$; the answer is $P_2(P_1X_0+Q_1)+Q_2$. \emph{Symmetric} (constraint
system): a ratio $A=NB$, a difference $C=B-K$, and a total $A+B+C$, solved jointly
for $A$. \emph{Hub} (indexed lookup): two rate tables, a colour-to-points map and
a size-to-multiplier map, and one inventory fragment listing several
(colour, size, count) items; the answer sums $\text{points}\cdot\text{mult}\cdot
\text{count}$ over the items. The headline of Table~\ref{tab:headline} pools three
data seeds ($42$, $7$, $100$), $n=100$ items per family per seed.

\paragraph{Natural-language stores.}
The two real-text sources are built to the same interface: an item's supporting
paragraphs become the store's fragments and its remaining paragraphs are held as
distractors, of which we merge in $d$. HotpotQA uses the \texttt{distractor}
validation split, keeping bridge questions with exactly two supporting paragraphs
(its ${\sim}8$ context paragraphs are the distractors). MuSiQue uses
\texttt{bdsaglam/musique} (answerable, validation), keeping the 2-hop subset with
two supporting paragraphs (its ${\sim}18$ non-supporting paragraphs are the
distractors, a harder locate test by construction). Both are scored by the
standard HotpotQA token-level $F_1$ and exact match, on free-form string answers.
We report $d=0$ (supporting only) and $d=8$.

\paragraph{Model and read.}
The main results use Qwen3-1.7B \citep{qwen3} with greedy decoding. Fragments are encoded
canonically and query-blind, each with $40$ latent reasoning steps (varied only in
Appendix~\ref{app:ablation}); the store merges by content hash into the
byte-identical render of Section~\ref{sec:store}. The read is the canonical
configuration throughout: route by the lexical signature (\texttt{token\_overlap},
the exact opaque tag set; the $128$-bit Bloom filter is an accelerator only,
Appendix~\ref{app:signatures}), decode each visited block under its mask into a
copy-prompted restatement of at most $64$ tokens, and compose. The $64$-token budget
is verified non-truncating against a $256$-token rerun (Section~\ref{sec:eval-nl}).
Cross-model (Llama) and larger scales (4B, 8B) are in Appendix~\ref{app:llama} and
Appendix~\ref{app:scale}.

The Llama replication ran under Python 3.11.2, Transformers 4.55.4,
\texttt{huggingface\_hub} 0.36.2, Tokenizers 0.21.4, and PyTorch
2.5.1+cu124 on Tesla V100-SXM2-32GB GPUs. Exact model revisions are reported in
Appendix~\ref{app:llama}; cache evidence and configuration-file checksums are
recorded in the accompanying results manifest and snapshot report.

\paragraph{Commands.}
The end-to-end tables come from \path{scripts/b2_self_extract.py}, run from
the \texttt{LatentMAS} directory. The designated mask and $k$-ramp diagnostics
come from \path{scripts/probe_block_masked.py}, and the illustrative anomaly
probe and routed fixed-rule probe are the separate scripts named in
Appendix~\ref{app:anomaly}. HMAC tag substitution is verified model-free to
leave exact routing unchanged; the key value therefore does not change these
results when writer and query client agree. The headline, for a single seed:
\begin{verbatim}
export MASREAD_SIGNATURE_KEY_HEX=<64-or-more-hex-characters>
python scripts/b2_self_extract.py --model_name Qwen/Qwen3-1.7B \
  --source k3 --families NEONYM,RELAY,CONSTRAINT,LOOKUP \
  --n_items 100 --data_seed 42 --canonical --query_blind \
  --extract_from masked_merged --selector token_overlap \
  --signature_encoding hmac --schedule traversal --restate_mode copy
\end{verbatim}

\paragraph{Artifact availability.}
The implementation, frozen result records, provenance evidence, environment
metadata, and CPU-only verification tools are archived in the \MaSRead{}
reproducibility artifact at
\url{https://doi.org/10.5281/zenodo.21459415}. The archive does not redistribute
model weights or raw dataset distributions; selected result records containing
HotpotQA or MuSiQue excerpts retain the corresponding dataset attribution and
license notices.

\bibliography{references}

\end{document}